\documentclass{article}

\usepackage{setspace}
\usepackage{amsmath}
\usepackage{amssymb}
\usepackage{amsthm}
\usepackage{vmr-symbols-vecbold}
\usepackage{standard-macros}
\usepackage{metanotation}
\usepackage[dvipsnames]{xcolor}
\usepackage{graphicx}
\usepackage{tikz}
\usetikzlibrary{patterns,tikzmark}
\usetikzlibrary{matrix,decorations.pathreplacing,calc}
\usepackage{hf-tikz}
\usetikzlibrary{decorations.pathreplacing,matrix}
\usepackage{paralist}
\usetikzlibrary{positioning,matrix}
\usetikzlibrary{decorations.pathreplacing,calligraphy}
\usepackage[mathscr]{euscript}

\usetikzlibrary{shapes.multipart}

     \usepackage[preprint, nonatbib]{neurips_2022}

\usepackage[utf8]{inputenc} %
\usepackage[T1]{fontenc}    %
\usepackage{hyperref}       %
\usepackage{url}            %
\usepackage{booktabs}       %
\usepackage{amsfonts}       %
\usepackage{nicefrac}       %
\usepackage{microtype}      %

\title{ Evaluated CMI Bounds for Meta Learning:\\
Tightness and Expressiveness }

\author{%
  Fredrik Hellstr\"om \\
  Chalmers University of Technology\\
  Gothenburg, Sweden  \\
  \texttt{frehells@chalmers.se} \\
   \And
     Giuseppe Durisi \\
  Chalmers University of Technology\\
  Gothenburg, Sweden \\
  \texttt{durisi@chalmers.se} \\
}

\begin{document}
\everypar{\looseness=-1}

\newcommand{\pop}[1]{\textcolor{OliveGreen}{#1}}

\newcommand{\popn}[1]{\textcolor{Apricot}{#1}}

\pgfkeys{tikz/mymatrixenv/.style={decoration={brace},every left delimiter/.style={xshift=8pt},every right delimiter/.style={xshift=-8pt}}}
\pgfkeys{tikz/mymatrix/.style={matrix of math nodes,nodes in empty cells,left delimiter={[},right delimiter={]},inner sep=1pt,outer sep=1.5pt,column sep=8pt,row sep=8pt,nodes={minimum width=20pt,minimum height=10pt,anchor=center,inner sep=0pt,outer sep=0pt}}}
\pgfkeys{tikz/mymatrixbrace/.style={decorate,thick}}
\pgfkeys{tikz/mymatrixbracethin/.style={decorate}}

\newcommand*\mymatrixbraceright[4][m]{
    \draw[mymatrixbrace] (#1.west|-#1-#3-1.south west) -- node[left=2pt, align = left] {#4} (#1.west|-#1-#2-1.north west);
}
\newcommand*\mymatrixbraceleft[4][m]{
    \draw[mymatrixbrace] (#1.east|-#1-#2-1.north east) -- node[right=2pt, align = center] {#4} (#1.east|-#1-#2-1.south east);
}
\newcommand*\mymatrixbracetop[4][m]{
    \draw[mymatrixbrace] (#1.north-|#1-1-#2.north west) -- node[above=2pt, align = center] {#4} (#1.north-|#1-1-#3.north east);
}
\newcommand*\mymatrixbracebottom[4][m]{
    \draw[mymatrixbrace] (#1.south-|#1-1-#2.north east) -- node[below=2pt, align = center] {#4} (#1.south-|#1-1-#3.north west);
}

\newcommand*\mymatrixbracetopShat[4][Shat]{
    \draw[mymatrixbrace] (#1.north-|#1-1-#2.north west) -- node[above=2pt, align = center] {#4} (#1.north-|#1-1-#3.north east);
}
\newcommand*\mymatrixbracebottomS[4][S]{
    \draw[mymatrixbrace] (#1.south-|#1-1-#2.north east) -- node[below=2pt, align = center] {#4} (#1.south-|#1-1-#3.north west);
}

\tikzset{style green/.style={
    set fill color=OliveGreen!50!lime!60,draw opacity=0.4,
    set border color=OliveGreen!50!lime!60,fill opacity=0.1,
  },
  style cyan/.style={
    set fill color=cyan!90!blue!60, draw opacity=0.4,
    set border color=blue!70!cyan!30,fill opacity=0.1,
  },
  style orange/.style={
    set fill color=orange!90, draw opacity=0.8,
    set border color=orange!90, fill opacity=0.3,
  },
  style brown/.style={
    set fill color=brown!70!orange!40, draw opacity=0.4,
    set border color=brown, fill opacity=0.3,
  },
  style purple/.style={
    set fill color=violet!90!pink!20, draw opacity=0.5,
    set border color=violet, fill opacity=0.3,    
  },
  style red/.style={
    set fill color=Red!90!pink!20, draw opacity=0.5,
    set border color=Red, fill opacity=0.3,    
  },
  Sh1/.style={
    above left offset={-0.08,0.25},
    below right offset={0.07,-0.16},
    #1
  },
  Sh/.style={
    above left offset={-0.07,0.25},
    below right offset={0.12,-0.16},
    #1
  },
  big/.style={
    above left offset={-0.01,0.375},
    below right offset={0.16,-0.35},
    #1
  },
  mid/.style={
    above left offset={-0.07,0.34},
    below right offset={0.12,-0.32},
    #1
  },
  small/.style={
    above left offset={-0.03,0.3},
    below right offset={0.09,-0.25},
    #1
  },set fill color/.code={\pgfkeysalso{fill=#1}},
  set border color/.style={draw=#1}
} 
\maketitle

\begin{abstract}
Recent work has established that the conditional mutual information (CMI) framework of Steinke and Zakynthinou (2020) is expressive enough to capture generalization guarantees in terms of algorithmic stability, VC dimension, and related complexity measures for conventional learning (Harutyunyan et al., 2021, Haghifam et al., 2021).
Hence, it provides a unified method for establishing generalization bounds.
In meta learning, there has so far been a divide between information-theoretic results and results from classical learning theory.
In this work, we take a first step toward bridging this divide.
Specifically, we present novel generalization bounds for meta learning in terms of the evaluated CMI (e-CMI).
To demonstrate the expressiveness of the e-CMI framework, we apply our bounds to a representation learning setting, with~$\taskin$ samples from~$\metan$ tasks parameterized by functions of the form~$f_i \circ h$.
Here, each~$f_i\in \mathcal F$ is a task-specific function, and~$h\in \mathcal H$ is the shared representation.
For this setup, we show that the e-CMI framework yields a bound that scales as~$\sqrt{ \mathcal C(\mathcal H)/(\taskin\metan) + \mathcal C(\mathcal F)/\taskin } $, where~$\mathcal C(\cdot)$ denotes a complexity measure of the hypothesis class.
This scaling behavior coincides with the one reported in Tripuraneni et al. (2020) using Gaussian complexity.
\end{abstract}

\section{Introduction}\label{sec:intro}
Meta learning, sometimes referred to as learning to learn, is a process by which performance on a new machine learning task is increased by using knowledge acquired from separate, but related, tasks \cite{caruana-97a,thrun-98a}.
Concretely, the meta learner~$\metalearner$ has access to training data from several different tasks, which are embedded in a common task environment, and aims to extract information from this data.
The goal is to use this information to improve the performance of a base learner~$\baselearner$ on a new task from the same task environment.
For instance, the task environment can consist of different image classification tasks, and the goal of the meta learner is to learn a shared representation for the tasks or to find suitable hyperparameters for a base learner performing image classification.

As in conventional learning, a central goal in meta learning is to bound the gap between the loss on the training data and the population risk on unseen data.
Two current approaches for achieving this goal are:
\begin{inparaenum}[i)]
    \item to use techniques from classical learning theory to obtain minimax performance guarantees, or
    \item to use information-theoretic methods to obtain algorithm-, data- and distribution-dependent guarantees.
    \end{inparaenum}
So far, these two lines of work have evolved largely separately.
In this paper, we take some steps toward unifying them.
Specifically, we:
\begin{inparaenum}[i)]
    \item derive new, tighter information-theoretic generalization bounds for meta learning, and
    \item demonstrate that these bounds are expressive enough to recover bounds for meta learning from classical learning theory.
    \end{inparaenum}
To concretize the discussion in this introduction, we assume that the meta learner outputs a member~$h$ of a function class~$\mathcal H$ on the basis of~$\taskin$ samples from~$\metan$ different tasks, and that a base learner selects a member~$f$ of a function class~$\mathcal F$, on the basis of the output of the meta learner and~$\taskin$ samples from a given task.

\paragraph{Classical learning theory for meta learning.}
The theoretical analysis of the benefits of meta learning in terms of loss bounds dates back to~\cite{baxter-20a}, where the notion of task environment was formally introduced.
More recently, for the setting of representation learning,~\cite[Thm.~5]{maurer-16a} derived a risk bound that scales as\footnote{In the interest of brevity, we suppress logarithmic factors throughout this section.}~$\sqrt{\mathcal C(\mathcal H)/\metan} + \sqrt{\mathcal C(\mathcal F)/\taskin}$, where~$\mathcal C(\cdot)$ denotes a complexity measure of the function class.
This demonstrates the benefit of meta learning for tasks that share a common environment.
Indeed, in the conventional single-task learning scenario, the~$\taskin$ samples from a given task need to be used for learning~$h$ and~$f$ simultaneously, leading to a~$\sqrt{\mathcal C(\mathcal H\times \mathcal F)/\taskin}$ bound.
The bound provided in~\cite[Thm.~5]{maurer-16a} was later improved by~\cite{tripuraneni-20a} to a scaling of~$\sqrt{\mathcal C(\mathcal H)/(\taskin\metan) + \mathcal C(\mathcal F)/\taskin}$.
This improved scaling, where~$\mathcal C(\mathcal{H})$ decays with the product~$\taskin\metan$, confirms the intuition that all of the~$\taskin\metan$ samples that are observed are informative at the environment level.
Meta learning has also been extensively studied in several special cases.
For instance, \cite{lounici-11a,cavallanti-10a,pontil-13a} study a setting with linear features and task mappings,
while~\cite{franceschi-18a,balcan-19a} consider an online convex optimization setting.
In this paper, we will mainly focus on the representation learning setting.

\paragraph{Information-theoretic generalization bounds.}
For conventional learning, the study of information-theoretic bounds was initiated by~\cite{russo-16a,xu-17a}, where the average generalization gap of a learning algorithm is bounded in terms of the information that the algorithm reveals about the training data.
At its heart, this line of work relies on a change of measure technique that relates the training loss to the population loss.
While the first information-theoretic bounds were given in terms of the mutual information between the output of the learning algorithm and the full training data, recent works provide bounds in terms of the \textit{disintegrated} mutual information between the \textit{loss} that the algorithm incurs on a \textit{single} sample pair and a selection variable indicating which sample is used for training, given a supersample containing both the training and test data.
These developments are due to the samplewise approach of~\cite{bu-20a}, the disintegration introduced in~\cite{negrea-19a}, the evaluated conditional mutual information (e-CMI) notion from~\cite{steinke-20a}, and combinations and extensions of these from~\cite{haghifam-20a,rodriguezgalvez-20a,hafezkolahi-20a,hellstrom-21b,harutyunyan-21a}.
This line of work is also intimately related to PAC-Bayesian generalization bounds~\cite{mcallester-98a,mcallester-13a}, where the generalization gap, averaged over the learning algorithm, is bounded with high probability over the data in terms of a KL divergence.
This is explored further in~\cite{hellstrom-20b,alquier-21a}.

\paragraph{Information-theoretic analysis of meta learning.}
Recently, information-theoretic generalization bounds have also been applied to meta learning~\cite{jose-21a,chen-21a,jose-22a}.
In parallel, a PAC-Bayesian analysis of meta learning has also been developed~\cite{pentina-14a,amit-18a,rothfuss-21a,guan-21a,flynn-22a,farid-21a}.
Generalization bounds obtained via information-theoretic methods have also been used as training objectives in order to improve performance~\cite{rothfuss-21a,yoon-18a}.
The quantity of interest in meta learning is the meta-population loss, which is the population loss evaluated on a task that was not observed during the meta learning phase.
While this quantity is unknown, the meta learner has indirect information about it through the observed meta-training loss, which is the loss that the meta learner incurs on the training samples from each of the observed tasks during the meta learning phase.
The standard approach in the information-theoretic and PAC-Bayesian analysis of meta learning consists of two steps.
The first step involves bounding the difference between the meta-training loss and a suitably defined auxiliary loss.
The second step involves bounding the difference between the meta-population loss and the auxiliary loss.
The two natural candidates for this auxiliary loss are the population loss of an observed task and the training loss for an unobserved task.
One of these steps (the first or second, depending on the choice of the auxiliary loss) is purely at the task level, while the other is purely at the environment level.
This makes it possible to view each of these steps as a conventional learning problem, so that a standard information-theoretic generalization bound can be applied for each step.
By the use of the triangle inequality, the two bounds are then combined to obtain a bound on the meta-population loss in terms of the meta-training loss.
We will refer to this procedure as a \textit{two-step} derivation.
An alternative approach was recently used by~\cite{chen-21a}, where a \textit{one-step} procedure was employed.
Rather than relying on an auxiliary loss,~\cite{chen-21a} immediately bounds the difference between the  meta-population loss and the meta-training loss in terms of a mutual information that captures both task level and environment level dependencies.
The environment and task level dependencies can then be obtained by decomposing this mutual information.
The resulting bound turns out to have a better scaling with~$\metan$ than the two-step bounds.
However, the information-theoretic analyses of meta learning reviewed so far do not provide any rigorous characterization of the scaling behavior of the bounds.
In particular, the dependence of the information measures on the sample size is typically ignored.
This precludes a direct comparison between these information-theoretic bounds and classical learning theory results.

\paragraph{Contributions. }
Focusing on the meta learning setup, we present novel information-theoretic bounds based on the e-CMI framework and demonstrate how to recover minimax results from classical learning theory via these bounds.
Our specific contributions are as follows.
In Section~\ref{sec:average_bounds}, we derive bounds for the average generalization error in terms of the disintegrated, samplewise e-CMI of the meta learner and base learner:
in Theorem~\ref{thm:two-step-sqrt-bound} and~\ref{thm:one-step-sqrt-bound}, we provide square-root bounds, which are shown to be tighter than results in the literature;
in Theorem~\ref{thm:avg-kl-bounds}, we derive novel bounds in terms of the binary KL divergence.
For low values of the training loss, the binary KL bounds display a more favorable dependence on the number of data samples than the square-root bounds.
Next, in Section~\ref{sec:high-probability}, we extend these average bounds to obtain high-probability generalization guarantees.
This is necessary to perform comparisons with high-probability bounds from classical learning theory.
Finally, in Section~\ref{sec:expressiveness}, we demonstrate the expressiveness of our bounds by applying them to a representation learning setting.
Under certain assumptions about the hypothesis classes, we provide upper bounds on the information measures that appear in our bound in terms of complexity measures.
The results that we obtain via this procedure display a scaling behavior that coincides with the one reported in~\cite{tripuraneni-20a}.
This demonstrates that the e-CMI framework is expressive enough to recover the scaling behavior of generalization guarantees for meta learning obtained via classical learning theory.

\section{Problem Setup and Notation}\label{sec:notation}

\begin{figure}
    \centering

\[
    \begin{tikzpicture}[text centered, baseline={-0.5ex},mymatrixenv]
     \matrix (m)  [mymatrix,inner sep=4pt] 
        {
    \tikzmarkin[big=style cyan]{Z1} \popn{\supersamplearg{1,0}_{1,0}}  &  \supersamplearg{1,0}_{1,1} & \tikzmarkin[mid=style brown]{Z11} \tikzmarkin[small=style purple]{Z112} \pop{\supersamplearg{1,1}_{1,0}}   &  \supersamplearg{1,1}_{1,1}   \tikzmarkend{Z112}  &  \supersamplearg{2,0}_{1,0}  &  \pop{\supersamplearg{2,0}_{1,1}} &  \supersamplearg{2,1}_{1,0}   &  \popn{\supersamplearg{2,1}_{1,1}} \\
    \supersamplearg{1,0}_{2,0}  &  \popn{\supersamplearg{1,0}_{2,1}} &  \pop{\supersamplearg{1,1}_{2,0}}   &  \supersamplearg{1,1}_{2,1}   &  \supersamplearg{2,0}_{2,0}  &  \pop{\supersamplearg{2,0}_{2,1}} &  \popn{\supersamplearg{2,1}_{2,0}}   &  \supersamplearg{2,1}_{2,1} \\
    \popn{\supersamplearg{1,0}_{3,0}}  &  \supersamplearg{1,0}_{3,1} &  \pop{\supersamplearg{1,1}_{3,0}}   &  \supersamplearg{1,1}_{3,1}   &  \pop{\supersamplearg{2,0}_{3,0}}  &  \supersamplearg{2,0}_{3,1} &  \popn{\supersamplearg{2,1}_{3,0}}   &  \supersamplearg{2,1}_{3,1} \\
    \supersamplearg{1,0}_{4,0}  &  \popn{\supersamplearg{1,0}_{4,1}} &  \supersamplearg{1,1}_{4,0}   &  \pop{\supersamplearg{1,1}_{4,1}} \tikzmarkend{Z1}\tikzmarkend{Z11} &  \supersamplearg{2,0}_{4,0}  &  \pop{\supersamplearg{2,0}_{4,1}} &  \popn{\supersamplearg{2,1}_{4,0}}   &  \supersamplearg{2,1}_{4,1} \\
    };
        \mymatrixbracebottom{4}{1}{Task pair \textcolor{Cerulean}{$\supersamplearg{1}$}}
        \mymatrixbracetop{3}{4}{In-task\\ supersample \textcolor{RawSienna}{$\supersamplearg{1,1}$}}
        \mymatrixbraceright{1}{1}{Sample \\ pair \textcolor{Purple}{$\supersamplearg{1,1}_{1}$} }
    \node[right = of m, xshift = 0.53cm] (ix) {$\supersamplearg{i,k}_{j,l}$} ;

    \node[above = of ix, xshift = 0.08cm, yshift = -1.35cm] (ixul) {};
    \node[above = of ix, xshift = -0.4cm, yshift = -1cm] (ixule) {};
    \draw [] (ixul) -- (ixule) ;
    \node[left = of ixule, xshift = 1.47cm, yshift = 0.11cm] {\footnotesize Task index} ;

    \node[above = of ix, xshift = 0.18cm, yshift = -1.35cm] (ixur) {};
    \node[above = of ix, xshift = 0.62cm, yshift = -1cm] (ixure) {};
    \draw [] (ixur) -- (ixure) ;
    \node[right = of ixure, xshift = -1.57cm, yshift = 0.11cm] {\footnotesize \begin{tabular}{l} Task \\ membership\end{tabular} } ;

    \node[below = of ix, xshift = 0.025cm, yshift = 1.35cm] (ixbl) {};
    \node[below = of ix, xshift = -0.42cm, yshift = 1cm] (ixble) {};
    \draw [] (ixbl) -- (ixble) ;
    \node[left = of ixble, xshift = 1.57cm, yshift = -0.1cm] {\footnotesize Sample index} ;

    \node[below = of ix, xshift = 0.11cm, yshift = 1.35cm] (ixbr) {};
    \node[below = of ix, xshift = 0.58cm, yshift = 1cm] (ixbre) {};
    \draw [] (ixbr) -- (ixbre) ;
    \node[right = of ixbre, xshift = -1.57cm, yshift = -0.1cm, align = left] {\footnotesize \begin{tabular}{l} Sample \\ membership\end{tabular} } ;

    \node[left = of m, xshift = 1cm] {$\metasupersample=$};

    \matrix (Shat) [mymatrix,inner sep=0pt, column sep = 2pt, below = of m, yshift = -0.7cm, xshift = -2.3cm]  
        {
     \tikzmarkin[Sh1=style green]{Shat1} 1 \tikzmarkend{Shat1}\\
      0 \\
    };
    \node[left = of Shat, xshift = 1.1cm, yshift = 0.06cm] {$\metasubsetchoice=$};

    \matrix (S) [mymatrix,inner sep=0pt, column sep = 2pt, right = of Shat, xshift=0.5cm, yshift = 0cm]  
        {
    1 & \tikzmarkin[Sh1=style green]{S1}  0 & 1 & 0 \\
    0 & 0 & 1 & 1 \\
    1 & 0 & 0 & 1 \\
    0 & 1 \tikzmarkend{S1}& 1 & 1 \\
    };
    \node[left = of S, xshift = 1.1cm, yshift = 0.06cm] {$\subsetchoice=$};

    \node[above = of S, xshift = -0.1cm, yshift = -1.1cm] (As) {};
    \node[above = of S, xshift = -0.7cm, yshift = -1.1cm] (Bs) {};
    \draw [decorate, thick,
    decoration = {brace,mirror}] (As) -- (Bs) ;
    \path (As) -- (Bs) node[midway, yshift = 0.4cm] (AsBsm) {\pop{$\subsetchoicearg{1,\metasubsetchoice_1}$}};

    \node[below = of S, xshift = -1.55cm, yshift = 1.1cm] (Asu) {};
    \node[below = of S, xshift = 1.55cm, yshift = 1.1cm] (Bsu) {};
    \draw [decorate, thick,
    decoration = {brace,mirror}] (Asu) -- (Bsu) ;
    \path (Asu) -- (Bsu) node[midway, yshift = -0.4cm] (AsBsmu) {Sample membership};

   \node[above = of Shat, xshift = -0.33cm, yshift = -1.1cm] (Ash) {};
    \node[above = of Shat, xshift = 0.33cm, yshift = -1.1cm] (Bsh) {};
    \draw [decorate, thick,
    decoration = {brace}] (Ash) -- (Bsh) ;
    \path (Ash) -- (Bsh) node[midway, yshift = 0.4cm] (AshBshm) {$\pop{\metasubsetchoice_1}$};

   \node[below = of Shat, xshift = -0.39cm, yshift = 1.1cm] (Ashu) {};
    \node[below = of Shat, xshift = 0.39cm, yshift = 1.1cm] (Bshu) {};
    \draw [decorate, thick,
    decoration = {brace, mirror}] (Ashu) -- (Bshu) ;
    \path (Ashu) -- (Bshu) node[midway, yshift = -0.55cm, align = center] (AshBshmu) {Meta\\membership};

    \matrix (Zss)  [mymatrix,inner sep=4pt, right = of S, xshift = 1.2cm, yshift = 0.3cm] 
        {
    \tikzmarkin[small=style green]{ZS1}  \pop{\supersamplearg{1,1}_{1,0}} & \pop{\supersamplearg{2,0}_{1,1}} \\
     \pop{\supersamplearg{1,1}_{2,0}} & \pop{\supersamplearg{2,0}_{2,1}} \\ \pop{\supersamplearg{1,1}_{3,0}} & \pop{\supersamplearg{2,0}_{3,0}} \\
     \pop{\supersamplearg{1,1}_{4,1}} \tikzmarkend{ZS1} & \pop{\supersamplearg{2,0}_{4,1}} \\
    };
    \node[left = of Zss, xshift = 1.1cm, yshift = 0.06cm] {$\metatrainset=$};

    \node[below = of Zss, xshift = 1.1cm, yshift = 1.1cm] (AZs) {};
    \node[below = of Zss, xshift = -1.1cm, yshift = 1.1cm] (BZs) {};
    \draw [decorate, thick,
    decoration = {brace}] (AZs) -- (BZs) ;
    \path (AZs) -- (BZs) node[midway, yshift = -0.4cm] (AZsBZsm) {Meta-training set};

    \node[above = of Zss, xshift = 0.03cm, yshift = -1.16cm] (AZsA) {};
    \node[above = of Zss, xshift = -0.98cm, yshift = -1.16cm] (BZsA) {};
    \draw [decorate, thick,
    decoration = {brace,mirror}] (AZsA) -- (BZsA) ;
    \path (AZsA) -- (BZsA) node[midway, yshift = 0.4cm] (AZsBZsAm) {$\pop{\supersamplearg{1,\metasubsetchoice_1}_ {\subsetchoicearg{1}}}$};

    \end{tikzpicture}
\]

    \caption{A graphical representation of our notation.
    In this example, the meta-supersample contains two task pairs:~$\supersamplearg{1}$, which is marked in~\textcolor{Cerulean}{blue}, and~$\supersamplearg{2}$.
    In turn,~$\supersamplearg{1}$ consists of the two in-task supersamples~$\supersamplearg{1,0}$ and~$\supersamplearg{1,1}$, which is marked in~\textcolor{RawSienna}{brown}.
    Next,~$\supersamplearg{1,1}$ consists of the four sample pairs~$\supersamplearg{1,1}_j$, where~$\,\, j=1,\dots,4$.
    In the figure,~$\supersamplearg{1,1}_1$ is marked in~\textcolor{Purple}{purple}.
    Finally,~$\smash{\supersamplearg{1,1}_1}$ consists of a pair of samples,~$\smash{\supersamplearg{1,1}_{1,0}}$ and~$\smash{\supersamplearg{1,1}_{1,1}}$.
    When a binary vector is used as subscript or superscript, this indicates that we should enumerate the meta-supersample according to that vector.
    To illustrate this, consider the construction of the meta-training set~$\metatrainset$.
    The meta-training set is formed on the basis of the meta-subset choice~$\metasubsetchoice$ and the observed-task subset choice~$\smash{\subsetchoicearg{\metasubsetchoice}}$.
    For instance, from the first task-pair,~$\metasubsetchoice_1=1$ indicates that we should select task~$1$.
    Then, from the first sample pair in this task,~$\smash{\subsetchoicearg{1,\metasubsetchoice_1}}=0$ indicates that we should select sample~$0$, which is~$\supersamplearg{1,1}_{1,0}$, marked in~\pop{green}.
    Repeating this for each sample pair in the in-task supersample~$\supersamplearg{1,1}$, we can identify the remaining elements of~$\smash{\supersamplearg{1,\metasubsetchoice_1}_ {\subsetchoicearg{1}}}$.
    This procedure is performed for each task index until all samples of~$\metatrainset$ are identified.
    The meta-test set~$\smash{\metatestset}$ is formed by an analogous procedure, but is now based on~$\compl{\metasubsetchoice}$ and~$\smash{\compl{\subsetchoicearg{\compl{\metasubsetchoice}}}}$.
    The entries from the meta-supersample that are selected for~$\metatestset$ are marked in~\popn{orange}.
    }
    \label{fig:supersample}
\end{figure}

We now introduce the meta learning setup that we consider throughout the paper, as well as the necessary notation for stating our results.
Similar to~\cite{baxter-20a}, we consider a task environment formulation that includes the representation learning setting of~\cite{maurer-16a,tripuraneni-20a} as a special case.

Our meta learning setup involves the following quantities.
We consider a task distribution~$\taskdistro$ on the task space~$\taskspace$.
For a given task~$\task\in\taskspace$, there is a corresponding in-task distribution~$\datadistro$ on the sample space~$\samplespace$.
The goal of the meta learner is to output a meta hypothesis~$\hyperparam \in \hyperspace$.
This is done on the basis of~$\taskin$ samples from~$\metan$ tasks.
Formally, the meta learner is a mapping~$\metalearner: \samplespace^{\taskin\times \metan}\times \metarandomspace \rightarrow \hyperspace$, where the random variable~$\metarandomness\in\metarandomspace$ captures the potential stochasticity of the learner.
The goal of the base learner is to output a hypothesis~$\paramarg{} \in\paramspace$, given the output of the meta learner and~$\taskin$ samples from a specific task.
Formally, the base learner is a mapping~$\baselearner: \samplespace^{\taskin}\times \randomspace\times \hyperspace\rightarrow \paramspace$.
The random vector~$\superrandomness\!=\!(\randomness_1,\dots,\randomness_{\metan})\!\in\!\randomspace^{\metan}$ has entries that capture the potential stochasticity of each base learner.
The entries are independent from the data and assumed to be identically distributed.\footnote{While identical distributions are not necessary for our results, this assumption simplifies the presentation.}
Here, the spaces~$\hyperspace$ and~$\paramspace$ may be function spaces or parameter spaces, depending on the learning algorithms.

Within each task, the training set for the base learner is randomly formed from a supersample according to the conditional mutual information (CMI) framework of~\cite{steinke-20a}.
Specifically, for a given task~$\tau$, let~$Z^{\tau}\!\in\! \samplespace^{\taskin\times 2}$ denote the supersample, which is an~$\taskin\!\times\!2$ matrix with elements generated independently from~$\datadistroarg{\tau}$.
For convenience, we index the two columns of~$Z^{\tau}$ by~$0$ and~$1$ and the rows by~$1,\!\dots\!,\taskin$.
The training set~$Z^{\tau}_S$ is formed on the basis of a membership vector~$S\!=\!(S_1,\dots,S_{\taskin})$, with entries generated independently from a~$\mathrm{Bern}(1/2)$ distribution.
More precisely, the~$j$th element of~$Z^{\tau}_S$ is given by~$[Z^{\tau}_S]_j\!=\!Z^{\tau}_{j,S_j}$, i.e., the~$S_j$th element from the~$j$th row of~$Z^{\tau}$.
Furthermore, we let~$\compl{S}\!=\!(1-S_1,\dots,1-S_{\taskin})$ denote the modulo-2 complement of~$S$, which we use to form the test set~$Z^{\tau}_{\compl{S}}$, whose~$j$th element is given by~$[Z^{\tau}_{\compl S}]_j\!=\!Z^{\tau}_{j,{\compl S}_j}$.
With this construction, we randomly assign each sample in the supersample to either the training set or test set with equal probability.

We now describe the meta-supersample~$\metasupersample$, which contains~$2\taskin$ samples from~$2\metan$ tasks, as in the meta-learning extension of the CMI framework provided in~\cite{rezazadeh-21a}.
Throughout, we let~$i\in\{1,\dots,\metan\}$ denote a task index,~$j\in\{1,\dots,\taskin\}$ denote a sample index, and~$k,l\in\{0,1\}$ denote binary indices indicating task membership and sample membership respectively.
Formally, the meta-supersample~$\metasupersample$ can be viewed as a data structure with~$4\taskin\metan$ elements.
In Figure~\ref{fig:supersample}, we illustrate~$\metasupersample$ as an~$\taskin\times 4\metan$ matrix for the case of~$\metan=2$ task pairs and~$\taskin=4$ sample pairs for each task.
We decompose~$\metasupersample$ as~$(\supersamplearg{1},\dots,\supersamplearg{\metan})$, where each element can be seen as a task pair.
Specifically, the pair~$\supersamplearg{i}$~can be decomposed as $(\supersamplearg{i,0},\supersamplearg{i,1})$, where each element is a task-specific supersample as described above.
The task-specific supersamples are~$\supersamplearg{i,k}=(\supersamplearg{i,k}_1,\dots,\supersamplearg{i,k}_{n})$, where each element is a pair of data samples.
Specifically, each sample pair is~$\supersamplearg{i,k}_{j}=(\supersamplearg{i,k}_{j,0},\supersamplearg{i,k}_{j,1})$, where~$\supersamplearg{i,k}_{j,l}\in \samplespace$.
The elements of $\metasupersample$ are generated as follows.
First, we generate $\taskik \distas \taskdistro$.
Then, we independently generate the samples~$\supersamplearg{i,k}_{j,l}\distas \datadistroik$.
This is repeated for all indices to form~$\metasupersample$.

\looseness = -1 Finally, we describe how the meta-training data is selected from the meta-supersample~$\metasupersample$.
This is done on the basis of the meta-membership vector~$\metasubsetchoice$ and the in-task membership vector~$\subsetchoice$.
Specifically, the meta-membership vector ~$\metasubsetchoice=(\metasubsetchoice_1,\dots,\metasubsetchoice_{\metan})$ is an~$\metan$-dimensional vector, while the sample membership vector~$\supersubsetchoice=(\subsetchoicearg{1,0},\subsetchoicearg{1,1},\dots,\subsetchoicearg{\metan,0},\subsetchoicearg{\metan,1})$ is a collection of~$2\metan$ vectors, where each~$\subsetchoiceik=(\subsetchoiceik_1,\dots,\subsetchoiceik_{\taskin})$ is an $\taskin$-dimensional vector.
The elements of all these vectors are generated independently from a~$\Bern{1/2}$ distribution.
For any Bernoulli matrix~$X$, we let~$\compl{X}$ denote its elementwise complement modulo~2, i.e.,~$\mathbf{1}-X$, where~$\mathbf{1}$ is the all-one matrix.

These membership vectors are used to form the meta-training set as follows.
We use the convention that, when a binary vector is used as a subscript or superscript of~$\supersample$, this indicates that we should enumerate over this vector.
Using this convention, the training set for the~$(i,k)$th task is constructed as~$\smash{\supersamplearg{i,k}_{\subsetchoiceik}=(\supersamplearg{i,k}_{\strut^{1,\subsetchoiceik_1}},\dots,\supersamplearg{i,k}_{\strut^{\taskin,\subsetchoiceik_{\taskin}}})}$.
We will use the shorthands~$\supersamplearg{i,k}_{\subsetchoicei}=\supersamplearg{i,k}_{\subsetchoiceik}$ and~${\supersamplearg{i,k}_{ {1,\subsetchoicei_j} }=\supersamplearg{i,k}_{ { {1,\subsetchoiceik_j}} }}$.
The test set for the~$(i,k)$th task~${\supersamplearg{i,k}_{{\compl{\subsetchoicei}}}}$ is constructed analogously, but on the basis of~$\compl{\subsetchoiceik}$.
The full\, meta-training set is~$\smash{\metatrainset=\!(\supersamplearg{\strut_{1,\metasubsetchoice_1}}_{\strut^{\subsetchoicearg{1,\metasubsetchoice_i}}},\dots,\supersamplearg{\strut_{\metan,\metasubsetchoice_{\metan}}}_{\strut^{\subsetchoicearg{\metan,\metasubsetchoice_{\metan}}}})}$, and the meta-test set~${\metatestset}$ is defined analogously.
With this construction, each task in the meta-supersample is assigned to either the meta-training set or the meta-test set with equal probability.
Then, as before, the samples within each task are assigned to an in-task training set or test set with equal probability.
The meta-training set consists of training samples within training tasks, while the meta-test set consists of test samples within test tasks.

We denote the output of the meta learner as~$U=\metalearner(\metatrainset,\metarandomness)\in\mathcal U$ and the output of the base learner for task~$(i,k)$ as~$\paramiarg{k}=\baselearner(\supersamplearg{i,k}_{\subsetchoicei},\randomness_i,U)\in \mathcal{W}$.
The performance of the learners is evaluated through a loss function~$\ell: \mathcal W \times \mathcal Z \rightarrow [0,1]$.
We denote the losses that the meta learner and base learner induce on the meta-supersample~$\metasupersample$ by~$\lambdahat$, which inherits the subscript and superscript notation that we described for~$\metasupersample$.
Thus, we have~$\lambdaarg{i,k}_{j,l} = \ell(\paramiarg{k} , \supersamplesmallarg{i,k}_{j,l} )$.
In other words,~$\lambdaarg{i,k}_{j,l}$ is the loss induced on the~$(j,l)$th sample in the~$(i,k)$th task.

On the basis of the loss matrix~$\metalambda$ and the membership vectors~$\metasubsetchoice$ and~$\subsetchoice$, we can compute four different losses.
The main quantity that we are interested in bounding is the average meta-population loss~$\avgmetapoploss$, which is the loss on test data for unobserved tasks.
The quantity that the meta learner has access to is the average meta-training loss,~$\avgtrainloss$, which is the training loss for observed tasks.
The other two losses are the average auxiliary test loss,~$\avgtestloss$, which is the loss on test data for observed tasks, and the average auxiliary training loss~$\avgunobstrainloss$, which is the loss on training data for unobserved tasks.
In the two-step derivations, one of these two quantities is used as the auxiliary loss.
These four losses are given by
\begin{align}
 \avgmetapoploss&=\avgij\Ex{\lambdai_j,\metasupersample}{\lambdaijnSinSj}, \qquad \avgtrainloss=\avgij\Ex{\lambdai_j,\metasupersample}{\lambdaijSiSj}  \\
  \avgtestloss&=\avgij\Ex{\lambdai_j,\metasupersample}{\lambdaijSinSj}  ,\qquad \avgunobstrainloss=\avgij\Ex{\lambdai_j,\metasupersample}{\lambdaijnSiSj}.
\end{align}

Finally, we end this section by introducing some information-theoretic quantities that appear in our bounds.
First, let~$P$ and~$Q$ be two probability measures such that~$P$ is absolutely continuous with respect to~$Q$.
The KL divergence between~$P$ and~$Q$ is denoted by~$\relent{P}{Q}$.
For the special case where~$P$ and~$Q$ are Bernoulli distributions with parameters~$p$ and~$q$, we let
\begin{equation}
\binrelent{p}{q}=\relent{P}{Q}=p\log\lefto(\frac q p\righto)+(1-p)\log\lefto(\frac{1-p}{1-q}\righto).
\end{equation}
We refer to~$\binrelent{p}{q}$ as the binary KL divergence.
The mutual information between the random variables~$X$ and~$Y$ is given by~$I(X;Y)=\relent{P_{XY}}{P_XP_Y}$, where~$P_{XY}$ is the joint distribution of~$X$ and~$Y$ and~$P_X$ and~$P_Y$ are the corresponding marginals.
The disintegrated mutual information between~$X$ and~$Y$ given a third random variable~$Z$ is given by~$I^{Z\!}(X;Y)=\relent{P_{XY\vert Z}}{P_{X\vert Z}P_{Y\vert Z}}$, where~$P_{XY\vert Z}$ is the conditional joint distribution of~$X$ and~$Y$ given~$Z$ and~$P_{X\vert Z}P_{Y\vert Z}$ is the product distribution formed from the corresponding marginals.
The expectation over~$Z$ of the disintegrated mutual information is the conditional mutual information~$I(X;Y\vert Z)=\Ex{Z}{I^{Z\!}(X;Y)}$.

\section{Generalization Bounds for Meta Learning with e-CMI}\label{sec:bounds}
In this section, we present generalization bounds in terms of the e-CMI of the meta learner and base learner.
In Section~\ref{sec:average_bounds}, we derive average square-root bounds that tighten results from~\cite{chen-21a,rezazadeh-21a}, as well as novel binary KL bounds.
In Section~\ref{sec:high-probability}, we extend these results to obtain novel, high-probability information-theoretic bounds for meta learning.
In Section~\ref{sec:expressiveness}, we demonstrate the expressiveness of the e-CMI framework by using the bounds from this section to recover generalization guarantees from classical learning theory for representation learning.

\subsection{Average Bounds}\label{sec:average_bounds}
In Theorem~\ref{thm:two-step-sqrt-bound}, we present a square-root bound for the average generalization error obtained through a two-step derivation.
Specifically, one step consists of bounding the unobserved training loss~$\avgunobstrainloss$ in terms of the observed training loss~$\avgtrainloss$, and the second step bounds the meta-population loss~$\avgmetapoploss$ in terms of~$\avgunobstrainloss$.
Chaining these two bounds, we obtain a bound on~$\avgmetapoploss$ in terms of~$\avgtrainloss$.
The bound depends on the information captured by two random variables: the task-level variable~$\lambdaarg{i,\compl{\metasubsetchoice_i}}_j$, which contains the training loss and test loss for task~$(i,\compl{\metasubsetchoice_i})$, as well as the environment-level variable~$\lambdai_{j,\subsetchoicei_j}$, which contains the training losses for both the observed task~$(i,\metasubsetchoice_i)$ and the unobserved task~$(i, \compl{\metasubsetchoice_i} )$.
We provide the proof of this result in Appendix~\ref{app:proofs}, along with the proofs of all other results in this paper.
\begin{thm}[Two-step square-root bound]\label{thm:two-step-sqrt-bound}
\thmstarttext
\begin{equation}\label{eq:thm-two-step-sqrt}
\abs{\avgmetapoploss\!-\!\avgtrainloss}\! \leq\! \avgij\! \Ex{\metasupersample,\subsetchoicei_j \!}{ {\sqrt{2I^{\metasupersample,\subsetchoicei_j}(\lambdai_{j,\subsetchoicei_j  };\metasubsetchoice_i)}} } \!+\! \avgij\! \Ex{\metasupersample,\metasubsetchoice_i\!}{\sqrt{2I^{\metasupersample,\metasubsetchoice_i}(\lambdaarg{i,\compl{\metasubsetchoice_i}}_j;\subsetchoicearg{i,\compl{\metasubsetchoice_i}}_j)}}\!.\!\!\!
\end{equation}
\end{thm}
The first term captures the environment-level generalization error while the second term captures the task-level generalization error.
In order to clarify the relation between Theorem~\ref{thm:two-step-sqrt-bound} and results from the literature, we relax it by upper-bounding the disintegrated individual-sample e-CMI terms by their integrated, full-sample, parametric CMI counterparts.
\begin{cor}\label{cor:two-step-sqrt-param}
Theorem~\ref{thm:two-step-sqrt-bound} implies that
\begin{equation}\label{eq:two-step-sqrt-param}
\abs{\avgmetapoploss - \avgtrainloss} \leq  \sqrt{  \frac{2I(U;\metasubsetchoice \vert \metasupersample,\subsetchoice )}{\metan}  }  +  \sqrt{\frac{2I(\paramarg{1,\compl{\metasubsetchoice_1}};\subsetchoicearg{1,\compl{\metasubsetchoice_1}} \vert \metasupersample,\metasubsetchoice_1 )}{\taskin}} .
\end{equation}
\end{cor}
This recovers the result of~\cite[Thm.~1]{rezazadeh-21a}, demonstrating that Theorem~\ref{eq:thm-two-step-sqrt} is tighter.

Next, we present an alternative square-root bound that is obtained through a one-step derivation.
This bound depends on the information captured by the random variable~$\lambdai_j$, which contains the training and test loss for both the observed and unobserved tasks.
\begin{thm}[One-step square-root bound]\label{thm:one-step-sqrt-bound}
\thmstarttext
\begin{align}\label{eq:thm-one-step-sqrt}
    \abs{\avgmetapoploss-\avgtrainloss}\leq \avgij \Ex{\metasupersample}{\sqrt{ 2 I^{\metasupersample}(\lambdai_j; \metasubsetchoice_i,\subsetchoicei_j)}}.
\end{align}
\end{thm}
Again, to compare this bound to results in the literature, we relax it by upper-bounding the disintegrated individual-sample e-CMI terms by their integrated, full-sample, parametric counterparts.
\begin{cor}\label{cor:one-step-sqrt-param} 
Let~$\paramarg{i}=(\paramiarg{0},\paramiarg{1})$ and~$W=\{\paramarg{i}\}_{i=1}^{\metan}$.
Then,
\begin{align}\label{eq:cor-one-step-sqrt-param}
\abs{\avgmetapoploss-\avgtrainloss}\leq  \sqrt{ \frac{2 I(W ; \metasubsetchoice,\subsetchoice \vert \metasupersample )}{\taskin\metan}} &\leq \sqrt{\frac{  2 I( \hyperparam; \metasubsetchoice,\subsetchoice\vert \metasupersample)   +  2\metan I( \paramarg{1}; \subsetchoicearg{1}\vert \metasupersample, \hyperparam)    }{\taskin\metan}} \\
    &\leq \sqrt{\frac{  2 I(\hyperparam;\metatrainset)   +  2\metan I( \paramarg{1}; \supersamplearg{1}_{\subsetchoicearg{1}}\vert \hyperparam)    }{\taskin\metan}}.\label{eq:cor-one-step-sqrt-param-cmi-to-mi}
\end{align}
\end{cor}
\looseness = -1 Up to some constant factors, this recovers the result in~\cite[Thm.~5.1]{chen-21a}.
Note that, if~$\metalearner$ or~$\baselearner$ are deterministic learning algorithms with continuous outputs, the mutual information terms in~\eqref{eq:cor-one-step-sqrt-param-cmi-to-mi} are unbounded.
In contrast, the CMI terms in~\eqref{eq:cor-one-step-sqrt-param} are always finite.
This is discussed in more detail in~\cite{steinke-20a}.
Furthermore, the bound in~\eqref{eq:cor-one-step-sqrt-param} compares favorably to~\cite[Thm.~1]{rezazadeh-21a}, since it decays with the product~$\taskin\metan$ rather than with~$\taskin$ and~$\metan$ separately.
This improvement is due to the one-step derivation.

Finally, in Theorem~\ref{thm:avg-kl-bounds}, we present two novel bounds in terms of the binary KL divergence.
The advantage of these bounds, as compared to the square-root bounds in Theorem~\ref{thm:two-step-sqrt-bound} and~\ref{thm:one-step-sqrt-bound}, is that they have a more favorable dependence on the number of samples for low training losses.
We demonstrate this improved rate for representation learning in Section~\ref{sec:expressiveness}.
\begin{thm}[Binary KL bounds]\label{thm:avg-kl-bounds}
For~$m\geq 2$,~$q\in[0,1]$ and~$c>0$, let
\begin{equation}\label{eq:dhalf-def}
\darg{m}(q,c)= \sup \bigg\{p\in[0,1]: \binrelent{q}{ \frac{q+p}{m} } \leq c\bigg\}.
\end{equation}
Then,
\begin{equation}\label{eq:thm-two-step-kl}
\avgmetapoploss \leq \darg{2}\lefto(\darg{2}\lefto( \avgtrainloss, \avgij I(\lambdai_{j,\subsetchoicei_j  };\metasubsetchoice_i\vert \metasupersample,\subsetchoicei_j) \righto), \avgij I(\lambdaarg{i,\compl{\metasubsetchoice_i}}_j;\subsetchoiceinSi_j\vert \metasupersample,\metasubsetchoice_i) \righto).
\end{equation}
Furthermore,
\begin{align}\label{eq:thm-one-step-kl} 
\avgmetapoploss + \avgtestloss + \avgunobstrainloss  \leq \darg{4}\lefto( \avgtrainloss, \avgij I(\lambdai_j;\metasubsetchoice_i,\subsetchoicei_j\vert \metasupersample) \righto).
\end{align}
\end{thm}
Interestingly,~\eqref{eq:thm-one-step-kl} provides a bound on the sum of the average meta-population loss~$\avgmetapoploss$, the test loss on observed tasks~$\avgtestloss$, and the training loss on unobserved tasks~$\avgunobstrainloss$.
Due to the nonnegativity of the loss, we can obtain an explicit bound on~$\avgmetapoploss$ by using the lower bound~$\avgtestloss+\avgunobstrainloss\geq 0$, which is a sensible relaxation when~$\avgmetapoploss$ is the dominant term. By this relaxation, we weaken the bound at most by a constant factor.
As previously mentioned, the bounds in Theorem~\ref{thm:avg-kl-bounds} can have a more favorable dependence on the number of samples than the square-root bounds in Theorem~\ref{thm:two-step-sqrt-bound} and~\ref{thm:one-step-sqrt-bound} when the training loss is low.
In the following corollary, we present a bound on~$\avgmetapoploss$ for the case where~$\avgtrainloss=0$.
\begin{cor}\label{cor:interpolating-one-step-kl}
Assume that~$\avgtrainloss=0$. Then, Theorem~\ref{thm:avg-kl-bounds} implies that
\begin{align}\label{eq:cor-interp-one-step-kl}
\avgmetapoploss \leq  4-4\exp\lefto(-\avgij I(\lambdai_j;\metasubsetchoice_i,\subsetchoicei_j\vert \metasupersample)\righto)\leq \frac4{\taskin\metan}\sumij I(\lambdai_j;\metasubsetchoice_i,\subsetchoicei_j\vert \metasupersample).
\end{align}
\end{cor}
Compared to the bound in Theorem~\ref{thm:one-step-sqrt-bound}, there is no square root in Corollary~\ref{cor:interpolating-one-step-kl}.
As we show in Section~\ref{sec:expressiveness}, this can lead to a faster rate of decay with the number of samples.

\subsection{High-probability Bounds}\label{sec:high-probability}
In the previous section, we provided bounds on the average generalization error.
However, meta learning bounds obtained via classical learning theory are typically high-probability bounds~\cite{maurer-16a,tripuraneni-20a}.
In order to assess the expressiveness of the e-CMI framework in terms of its ability to recover these results, we now extend the bounds from Section~\ref{sec:average_bounds} to the high-probability setting.
For this, we need some additional notation.
We let~$\pacbavgmetapoploss$ and~$\pacbavgtrainloss$ denote the meta-population loss and training loss given that the meta-training set is constructed from~$\pacbtraindata$. Specifically,
\begin{align}
\pacbavgmetapoploss &= \avgij \Ex{\metarandomness,\randomness_i}{\ell(\baselearner(\supersamplearg{i,\compl{\metasubsetchoice_i}}_{\subsetchoicei},\randomness_i,\metalearner(\metatrainset,\metarandomness)) , \supersamplearg{i,\compl{\metasubsetchoice_i}}_{j,{\compl{\subsetchoicei_j}} } )}, \\
\pacbavgtrainloss &= \avgij \Ex{\metarandomness,\randomness_i}{\ell(\baselearner(\supersamplearg{i,\metasubsetchoice_i}_{\subsetchoicei},\randomness_i,\metalearner(\metatrainset,\metarandomness)) , \supersamplearg{i,\metasubsetchoice_i}_{j,{\subsetchoicei_j} } )} .
\end{align}

We now present a high-probability version of the two-step square root bound in Theorem~\ref{thm:two-step-sqrt-bound}.
To simplify the presentation, we omit explicit constants and assume that~$\metalearner$ and~$\baselearner$ are indifferent to the order of the data samples.
The theorem statement is provided in more general form in Appendix~\ref{app:proofs}.

\begin{thm}[High-probability two-step square-root bound]\label{thm:high-prob-two-step-sqrt}
Let~$Q_{1,\subsetchoice_1}$ denote the conditional distribution of~$\lambda_{1,\subsetchoice_1}$ given~$(\metasupersample, \metasubsetchoice,\subsetchoice )$, and let~$P_{1,\subsetchoice_1}$ denote~$\Ex{\metasubsetchoice}{Q_{1,\subsetchoice_1}}$.
Furthermore, let~$Q^{1,\compl{\metasubsetchoice_1}}$ denote the conditional distribution of~$\lambdaarg{1,\compl{\metasubsetchoice_1}}$ given~$(\metasupersample, \metasubsetchoice_1, \subsetchoicearg{1,\compl{\metasubsetchoice_1}})$, and let~$P^{1,\compl{\metasubsetchoice_1}}$ denote~$\Ex{\subsetchoicearg{1,\compl{\metasubsetchoice_1}}}{Q^{1,\compl{\metasubsetchoice_1}}}$.
Then, there exist constants~$C_1,C_2$ such that, with probability at least~$1-\delta$ under the draw of~$(\metasupersample,\metasubsetchoice,\subsetchoice)$,
\begin{multline}\label{eq:thm-high-prob-two-step-sqrt}
\abs{  \pacbavgmetapoploss  -  \pacbavgtrainloss  }   \leq    C_1{\sqrt{  \frac{\relent{Q_{1,\subsetchoice_1}  }{  P_{1,\subsetchoice_1}  }  +  \log( \frac{\taskin\sqrt{\metan}}{\delta} )}{\metan}}} \\
+  C_2 \sqrt{\frac{  \relent{Q^{1,\compl{\metasubsetchoice_1}}  }{  P^{1,\compl{\metasubsetchoice_1}}  }
  +  \log( \frac{\sqrt{\taskin}}{\delta} )}{ \taskin }  }.  
\end{multline}
\end{thm}
The KL divergences in~\eqref{eq:thm-high-prob-two-step-sqrt} can be interpreted as pointwise e-CMIs.
Indeed,
\begin{align}
    \Ex{\pacbtraindatan}{ \relent{Q_{1,\subsetchoice_1}}{ P_{1,\subsetchoice_1} }} &= I( \lambda_{1,\subsetchoice_1} ; \metasubsetchoice  \vert  \subsetchoice,\metasupersample  ), \\
    \Ex{\pacbtraindatan}{\relent{Q^{1,\compl{\metasubsetchoice_1}}}{ P^{1,\compl{\metasubsetchoice_1}} }) } &=  I(\lambdaarg{1,\compl{\metasubsetchoice_1}} ; \subsetchoicearg{1,\compl{\metasubsetchoice_1}} \vert  \metasupersample,\metasubsetchoice_1 ) .
\end{align}

Finally, we present a high-probability version of the one-step square root bound in Theorem~\ref{thm:one-step-sqrt-bound}.

\begin{thm}[High-probability one-step square-root bound]\label{thm:high-prob-one-step-sqrt}
Let~$Q$ denote the conditional distribution of~$\lambdahat$ given~$(\metasupersample,\metasubsetchoice,\subsetchoice)$, and let~$P$ denote~$\Ex{\metasubsetchoice,\subsetchoice}{Q}$.
Then, with probability at least~$1-\delta$ under the draw of~$(\metasupersample,\metasubsetchoice,\subsetchoice)$,
\begin{equation}\label{eq:thm-high-prob-one-step-sqrt}
\abs{\avgmetapoploss(\metasupersample,\metasubsetchoice,\subsetchoice)\!-\!\avgtrainloss(\metasupersample,\metasubsetchoice,\subsetchoice)}\!\\  \leq\! \!  {\sqrt{ \frac{2\left(\relent{Q}{ P } + \log( \frac{\sqrt{\taskin\metan}}{\delta} )\right)}{\taskin\metan-1}}} .
\end{equation}
\end{thm}
Again, the KL divergence can be interpreted as a pointwise e-CMI, since
\begin{equation}
    \Ex{\pacbtraindatan}{\relent{Q}{ P }} = I(\lambdahat;\metasubsetchoice,\supersubsetchoice\vert \metasupersample).
\end{equation}
\section{Expressiveness of the Bounds}\label{sec:expressiveness}
In Section~\ref{sec:bounds}, we presented several new information-theoretic generalization bounds, and demonstrated that they improve upon known bounds from the literature.
We now turn our focus to the expressiveness of the e-CMI framework.
In particular, we show that the bounds from Section~\ref{sec:bounds} can be used to recover generalization guarantees for meta learning from classical learning theory.
Specifically, we consider the representation learning setting that is analyzed in~\cite{tripuraneni-20a}.
We use the following notation.
First, the sample space is the product of an instance space and label space:~$\samplespace=\examplespace\times \trainlabelspace$.
The aim of the meta learner is to find a representation~$h_\hyperparam\in\mathcal H$, while the base learner outputs a task-specific function~$f_W\in\funcspace$.
Composing these functions, we obtain the mapping~$f_W\circ h_U:\examplespace\rightarrow\trainlabelspace$.

\subsection{Minimax Generalization Bounds}\label{sec:expressiveness-genbounds}

To obtain explicit minimax bounds, we assume that~$\mathcal H$ has finite Natarajan dimension~$d_N$ and that~$\mathcal F$ has finite VC dimension~$\dVC$.
This allows us to derive bounds on the entropy of the representations and predictions that the meta learner and base learner induce on the meta-supersample.
This, in turn, leads to bounds on the e-CMI terms that appear in the bounds in Section~\ref{sec:bounds}.
In the following corollary, we present the bounds that are obtained by bounding the e-CMI terms in Theorem~\ref{thm:two-step-sqrt-bound} and~\ref{thm:one-step-sqrt-bound}.
These hold for any learner that outputs hypotheses from the specified classes.
\begin{cor}\label{cor:rep-learning-bound}
Assume that the range of~$\mathcal H$ has cardinality~$N$, that the Natarajan dimension of~$\mathcal H$ is~$d_N$, and that the VC dimension of~$\mathcal F$ is~$\dVC$.
Also, let~$2\taskin\geq \dVC+1$ and~$2\metan\geq d_N+1$.
Then,
\begin{equation}\label{eq:cor-rep-learning-bound-two-step}
\abs{\avgmetapoploss\!-\!\avgtrainloss} \leq\
\sqrt{ \frac{ 2d_N\log \lefto(\binom{N}{2}\frac{2e\metan}{d_N} \righto)}{\metan} } + \sqrt{ \frac{ 2\dVC\log \lefto(\frac{2e\taskin}{\dVC} \righto) }{\taskin} },
\end{equation}
\begin{equation}\label{eq:cor-rep-learning-bound-one-step}
\abs{\avgmetapoploss-\avgtrainloss} \leq \sqrt{ \frac{2d_N \log \lefto(\binom{N}{2}\frac{4e\taskin\metan}{d_N} \righto) + 4\metan\dVC \log \lefto(\frac{2e\taskin}{\dVC} \righto)}{\taskin\metan} }.
\end{equation}
\end{cor}
Corollary~\ref{cor:rep-learning-bound} establishes that, for the average setting, we can use the bounds in Theorems~\ref{thm:two-step-sqrt-bound} and~\ref{thm:one-step-sqrt-bound} to obtain minimax bounds for function classes with bounded Natarajan and VC dimensions.
Note that, in the upper-bound of~\eqref{eq:cor-rep-learning-bound-two-step}, we have fully decoupled the complexity of the two function classes.
This is made possible by the fact that we used~$\avgunobstrainloss$ as the auxiliary loss in the derivation of Theorem~\ref{thm:two-step-sqrt-bound}, rather than~$\avgtestloss$.
We discuss this in more detail in Appendix~\ref{app:proofs}.

Next, we consider the interpolating setting, where~$\avgtrainloss=0$.
Under this assumption, we demonstrate that we can achieve a better rate of convergence with respect to the number of training samples.
The result, presented in the following corollary, relies on similarly bounding the e-CMI term in Corollary~\ref{cor:interpolating-one-step-kl}.
\begin{cor}\label{cor:rep-learning-bound-interp}
Consider the setting of Corollary~\ref{cor:rep-learning-bound}.
Furthermore, assume that~$\avgtrainloss=0$.
Then,
\begin{align}\label{eq:cor-rep-learning-bound-interp}
\avgmetapoploss  \leq \frac{4d_N \log \lefto(\binom{N}{2}\frac{4e\taskin\metan}{d_N} \righto) + 8\metan\dVC \log \lefto(\frac{2e\taskin}{\dVC} \righto) }{\taskin\metan}.
\end{align}
\end{cor}
The result in Corollary~\ref{cor:rep-learning-bound-interp} demonstrates that, for the interpolating setting, the e-CMI framework is expressive enough to yield a bound that, ignoring logarithmic factors, decays as~$1/(\taskin\metan)$, often referred to as a fast rate.

Finally, noting that the bounds in~\cite{maurer-16a} and~\cite{tripuraneni-20a} are high-probability rather than average bounds, we also derive high-probability generalization bounds.
In order to achieve this, we need probabilistic upper bounds on the KL divergences that appear in Theorem~\ref{thm:high-prob-two-step-sqrt} and~\ref{thm:high-prob-one-step-sqrt}, similar to how the e-CMI terms were bounded for Corollary~\ref{cor:rep-learning-bound} and~\ref{cor:rep-learning-bound-interp}.
The resulting bounds are presented in the following corollary.
\begin{cor}\label{cor:rep-learning-high-probability}
Consider the setting of Corollary~\ref{cor:rep-learning-bound}.
Then, there exist constants~$C_1,C_2,C_3$ such that, with probability at least~$1-\delta$ under the draw of~$\pacbtraindata$,
\begin{equation}\label{eq:high-probability-cor-rep-learning-bound-two-step}
\abs{\avgmetapoploss\pacbtraindatash\!-\!\avgtrainloss\pacbtraindatash}\! 
\!\leq\! C_1\!
\sqrt{ \frac{\! \!d_N\log \lefto(\!\binom{N}{2}\frac{\metan}{d_N} \!\righto) \! \!+\!  \log\lefto( \!\frac{\taskin\sqrt{\metan}}{\delta}\!\righto) }{\metan}  } \!+\! C_2\sqrt{ \frac{ \!\!\dVC\log \lefto(\!\frac{\taskin}{\dVC} \!\righto) \! \!+\! \log\lefto(\! \frac{\sqrt{\taskin}}{\delta}\!\righto) }{\taskin} },\!\!\!
\end{equation}
\begin{equation}\label{eq:high-probability-cor-rep-learning-bound-one-step}
\abs{\avgmetapoploss\pacbtraindata-\avgtrainloss\pacbtraindata} \leq C_3\sqrt{ \frac{d_N \log \left(\binom{N}{2}\frac{\taskin\metan}{d_N} \right) + \metan\dVC \log \left(\frac{\taskin}{\dVC} \right) + \log\lefto( \frac{\sqrt{\taskin\metan}}{\delta}\righto)}{\taskin\metan} }.
\end{equation}
\end{cor}

We now see that, suppressing logarithmic factors, the upper bound in~\eqref{eq:high-probability-cor-rep-learning-bound-two-step-excess} scales as~$\twostepscaling$, whereas the upper bound in~\eqref{eq:high-probability-cor-rep-learning-bound-one-step} scales as~$\onestepscaling$. This matches the rates obtained by~\cite{maurer-16a} and~\cite{tripuraneni-20a}, respectively, demonstrating that the e-CMI framework, combined with the one-step approach, is expressive enough to recover the scaling of these results.

Note that there are some differences between these results and the ones in~\cite{maurer-16a,tripuraneni-20a}.
First, while the complexity measures that we use are related to the Natarajan and VC dimension, the results in~\cite{tripuraneni-20a} are given in terms of Gaussian complexity.
Furthermore, while~\cite{maurer-16a,tripuraneni-20a} provide excess risk bounds for a fixed target task with~$m$ training samples, the bounds in Corollary~\ref{cor:rep-learning-high-probability} are generalization bounds for a randomly drawn task.
In Section~\ref{sec:excess-informal}, we extend our analysis to derive excess risk bounds for a fixed target task.

\subsection{Excess Risk Bounds}\label{sec:excess-informal}

In order to derive excess risk bounds for a specific target task, as is done in~\cite{tripuraneni-20a}, we need to assume that the meta learner~$\metalearner$ and the base learner~$\baselearner$ are empirical risk minimizers.
This is in contrast to all previous bounds in this paper, which apply to any learning algorithms.
Furthermore, we need a notion of oracle algorithms, which minimize the population loss.
Finally, we need to assume that the tasks contained in the meta-supersample satisfy a notion of \textit{task diversity}. 
Intuitively, this means that, given the output of the empirical risk-minimizing meta learner, the performance of the oracle base learner on the tasks in the meta-supersample gives a reasonable indication of the performance of the oracle base learner on any possible task.
Due to space constraints, we state here an informal version of a high-probability excess risk bound for a specified target task based on the one-step square-root generalization bound in Corollary~\ref{cor:rep-learning-high-probability}.
A precise statement of this result, along with its proof, is given in Appendix~\ref{app:excess}.

\begin{manualcor}{7 }[Informal]
Consider the setting of Corollary 6 and a fixed task~$\tau_0$.
Let~$Z^0_{S^0}\in \mathcal Z^{ m}$ be a vector of~$m$ samples generated independently according to the data distribution~$\datadistroarg{\tau_0}$ for task~$\tau_0$. 
Let~$\metalearner$ and~$\baselearner$ be empirical risk minimizers.
Let~$L_0\pacbtraindataandzeros$ denote the population loss on task~$\tau_0$ when applying~$\metalearner$ to~$\metatrainset$ and~$\baselearner$ to~$(Z^0_{S^0},\metalearner(\metatrainset))$, and let~$L_0^*$ denote the smallest population loss for task~$\tau_0$ that can be obtained using functions from~$\mathcal{F}$ and~$\mathcal{H}$.
Finally, assume that the supersample satisfies a task-diversity assumption with parameters~$\nu,\epsilon$.
Then, there exist constants~$C_1$ and~$C_2$ such that, with probability at least~$1-\delta$ under the draw of~$(\pacbtraindatan,Z^0_{S^0})$,
\begin{multline}\label{eq:high-probability-cor-rep-learning-bound-two-step-excess}
L_0(\pacbtraindatan,Z^0_{S^0}) - L_0^*
\leq C_1\sqrt{ \frac{ \dVC \log\lefto(\frac{\sqrt m}{\dVC}\righto) \!+\! \log\lefto(\frac{\sqrt m}{\delta}\righto) }{m} } \\
+ C_2 \nu^{-1}\sqrt{ \frac{d_N \log \left(\binom{N}{2}\frac{\taskin\metan}{d_N} \right) + \metan\dVC \log \left(\frac{\taskin}{\dVC} \right) + \log\lefto( \frac{\sqrt{\taskin\metan}}{\delta}\righto)}{\taskin\metan} } + \epsilon.
\end{multline}
\end{manualcor}

The bound in~\eqref{eq:high-probability-cor-rep-learning-bound-two-step-excess} displays the same scaling as the excess risk bound in~\cite{tripuraneni-20a}.

It is possible to derive high-probability bounds based on the two-step square-root generalization bound in Corollary~\ref{cor:rep-learning-high-probability} by suitably substituting the two-step bound in the proof of Corollary~\ref{cor:excess-bound-spec-task}.
The same can be done using the bounds that are given in terms of information measures, and average excess risk bounds can also be derived by an analogous procedure.
Finally, we note that it is possible to derive excess risk bounds for a new, random task, rather than a specified target task, without assuming task diversity.
This is done in Corollary~\ref{cor:excess-random-task} in Appendix~\ref{app:excess}.

\section{Conclusions}\label{sec:conclusions}
In this paper, we derived new generalization bounds for meta learning using e-CMI, which improve upon information-theoretic bounds found in the literature.
By considering a representation learning setting, we demonstrated that e-CMI bounds obtained via a conventional two-step approach lead to rates that coincide with those found in~\cite{maurer-16a}.
In contrast, we showed that by combining the e-CMI framework with a one-step approach, we recover the more favourable scaling found in~\cite{tripuraneni-20a}.
Note that, while the bounds in~\cite{tripuraneni-20a} are uniform over the hypothesis class, the information-theoretic bounds that we derive are inherently algorithm- and data-dependent.
As a consequence, they are nonvacuous when applied to settings such as classification with deep neural networks~\cite{harutyunyan-21a}.
The algorithm-dependence and expressiveness of our bounds indicate that they can be developed further to guide algorithm design.
However, no recipe for this is provided in this paper.
It should also be noted that the complexity measures that we consider differ from the Gaussian complexity in~\cite{tripuraneni-20a}.
An intriguing topic for further study is to clarify the connection between e-CMI and Gaussian complexity.

\section*{Acknowledgements}

This work was partly supported by the Wallenberg AI, Autonomous Systems and Software Program (WASP) funded by the Knut and Alice Wallenberg Foundation and the Chalmers AI Research Center (CHAIR).

%

%

%
%
%
%
%
%
%
%
%

%
%

\bibliography{reference}
\bibliographystyle{unsrt}

%
\appendix

\section{Proofs}\label{app:proofs}
In this appendix, we present the proofs of the results in the main paper.
First, we give a summary of the notation that is used in this appendix.
Then, in Section~\ref{sec:useful-lemmas}, we present some lemmas that are useful for proving our main results.
In Section~\ref{sec:proof-average}, we prove the average generalization bounds from Section~\ref{sec:average_bounds}.
In Section~\ref{sec:proof-high-probability}, we prove the high-probability results from Section~\ref{sec:high-probability}.
Finally, in Section~\ref{sec:proof-expressive}, we prove the generalization bounds for multiclass classification from Section~\ref{sec:expressiveness-genbounds}.

\paragraph{Notation summary.} For~$i\in\{1,\dots,\metan\}$,~$k\in\{0,1\}$,~$j\in\{1,\dots,\taskin\}$, and~$l\in\{0,1\}$, we let~$\supersamplearg{i,k}_{j,l}$ denote the~$l$th sample from the~$j$th sample pair in the~$k$th task of the~$i$th task pair.
This is illustrated in Figure~\ref{fig:supersample}.
Throughout,~$i$ denotes a task index,~$j$ denotes a sample index,~$k$ denotes a selection within the task pair, and~$l$ denotes a selection within the sample pair.
Furthermore, we let~$\supersamplearg{i}_{j,l}=( \supersamplearg{i,0}_{j,l}, \supersamplearg{i,1}_{j,l} )$ and~$\supersamplearg{i}_{j}=\{ \supersamplearg{i,k}_{j,l} \}_{l=0,1}^{k=0,1}$.
The tasks used to form the training set are selected on the basis of the binary vector~$\metasubsetchoice=(\metasubsetchoice_1,\dots,\metasubsetchoice_{\metan})$.
Within task~$(i,k)$, the samples that form the training set are selected on the basis of~$\subsetchoiceik=(\subsetchoiceik_1,\dots,\subsetchoiceik_{\taskin})$.
For convenience, we let~$\subsetchoicei=(\subsetchoicearg{i,0},\subsetchoicearg{i,1})$ and~$\subsetchoice=(\subsetchoicearg{1},\dots,\subsetchoicearg{\metan})$.
The training set for task~$(i,k)$ is~$\supersamplearg{i,k}_{\subsetchoicearg{i,k}}=(\supersamplearg{i,k}_{1,\subsetchoiceik_1},\dots,\supersamplearg{i,k}_{\taskin,\subsetchoiceik_{\taskin}})$.
As a shorthand,~$\supersamplearg{i,k}_{\subsetchoicearg{i}}=\supersamplearg{i,k}_{\subsetchoicearg{i,k}}$.
The collection of all samples is~$\metasupersample=\{ \supersamplearg{i}_j \}_{j=1:\taskin}^{i=1:\metan}$.
The full data set for task~$(i,k)$ is~$\supersamplearg{i,k}=\{ \supersamplearg{i,k}_{j,l} \}_{j=1:\taskin}^{l=0,1}$.
The full data set for all training tasks is~$\supersamplearg{\metasubsetchoice}=(\supersamplearg{1,\metasubsetchoice_1},\dots,\supersamplearg{\metan,\metasubsetchoice_{\metan}} )$.
The~$j$th training sample for task pair~$i$ is~$\supersamplearg{i}_{j,\subsetchoice_j}=(\supersamplearg{i,0}_{j,\subsetchoice_j},\supersamplearg{i,1}_{j,\subsetchoice_j})$.
The~$j$th training sample for all tasks is~$\supersamplearg{}_{j,\subsetchoice_j}=(\supersamplearg{1}_{j,\subsetchoice_j},\dots,\supersamplearg{\metan}_{j,\subsetchoice_j})$.
The training sets for all tasks is~$\supersamplearg{}_{\subsetchoice}=(\supersamplearg{}_{1,\subsetchoice_1},\dots,\supersamplearg{}_{\taskin,\subsetchoice_{\taskin}})$.
The meta-training set is~$\metatrainset=(\supersamplearg{1,\metasubsetchoice_1}_{\subsetchoicearg{1}},\dots,\supersamplearg{\metan,\metasubsetchoice_1}_{\subsetchoicearg{\metan}})$. Finally, the output of the meta learner is~$\hyperparam$, the output of the base learner for task~$(i,k)$ is~$\paramarg{i,k}$, and we let~$\paramarg{i}=(\paramarg{i,0},\paramarg{i,1})$ and~$W=(\paramarg{1},\dots,\paramarg{\metan})$.

The conventions that we describe for~$\supersamplearg{i,k}_{j,l}$ apply also for the losses~$\lambdaarg{i,k}_{j,l}$, the instances~$\superexamplearg{i,k}_{j,l}$, the predictions~$\Farg{i,k}_{j,l}$, and the representations~$\Harg{i,k}_{j,l}$ that we consider in this appendix.

\subsection{Useful Lemmas}\label{sec:useful-lemmas}
In this section, we present some lemmas that will be useful in the derivations of the main results.
We begin with two change of measure inequalities for functions of random variables.

\begin{lem}[Change of measure inequalities]
Let~$X$ and~$Y$ be two random variables over~$\mathcal X$ and~$\mathcal Y$ respectively, and let~$Y'$ be a random variable with the same marginal distribution as~$Y$ such that~$Y'$ and~$X$ are independent.
Assume that the joint distribution of~$X,Y$ is absolutely continuous with respect to the joint distribution of~$X,Y'$.
Let~$f:\mathcal X\times \mathcal Y\rightarrow [-1,1]$ and~$g:\mathcal X\times\mathcal Y\rightarrow [0,1]$ be measurable functions.
Furthermore, assume that~$\Ex{X,Y'}{f(X,Y')}=0$.
Then, the following inequalities hold:
\begin{align}
 \abs{\Ex{X,Y}{f(X,Y)}} &\leq \sqrt{2I(X;Y)} ,\label{eq:lem_subgauss_com}\\
 \binrelent{\Ex{X,Y}{g(X,Y)}}{\Ex{X,Y'}{g(X,Y')}} &\leq I(X;Y).\label{eq:lem_kl_com}
\end{align}
\end{lem}
\begin{proof}
Donsker-Varadhan's variational representation of the KL divergence implies that
\begin{equation}\label{eq:com-pf-step-1}
\Ex{X,Y}{\gamma f(X,Y)} \leq I(X;Y) + \log \Ex{X,Y'}{ e^{\gamma f(X,Y')} }.
\end{equation}
Now,~$f(X,Y')$ is bounded to~$[-1,1]$ and~$\Ex{X,Y'}{f(X,Y')}=0$.
Therefore,~$f(X,Y')$ is a sub-Gaussian random variable, which implies that
\begin{equation}
\log \Ex{X,Y'}{ e^{\gamma f(X,Y')} } \leq \gamma^2/2.
\end{equation}
Using this upper bound in~\eqref{eq:com-pf-step-1}, we obtain
\begin{equation}
\Ex{X,Y}{\gamma f(X,Y)} -  \gamma^2/2 \leq I(X;Y).
\end{equation}
By maximizing the left-hand side over~$\gamma$, we establish~\eqref{eq:lem_subgauss_com}.

We now turn to~\eqref{eq:lem_kl_com}.
Let~$\binrelentg{q}{p}=\gamma q-\log(1-p+pe^\gamma)$,
and note that this function is convex.
By Jensen's inequality,
\begin{equation}\label{eq:com-pf-step-kl-1}
\binrelentg{\Ex{X,Y}{g(X,Y)}}{\Ex{X,Y'}{g(X,Y')}} \leq \Ex{X,Y}{\binrelentg{g(X,Y)}{\Ex{Y'}{g(X,Y')}}} .
\end{equation}
By Donsker-Varadhan's variational representation of the KL divergence,
\begin{equation}%
\Ex{X,Y}{\binrelentg{g(X,Y)}{\Ex{Y'}{g(X,Y')}}} \leq I(X;Y) + \log \Ex{X,Y'}{ e^{ \binrelentglr{g(X,Y')}{\Ex{Y'}{g(X,Y')}} } }.
\end{equation}
By~\cite[Eq.~(17)]{mcallester-13a}, we have
\begin{equation}\label{eq:com-pf-step-kl-last}
\log \Ex{X,Y'}{ e^{ \binrelentglr{g(X,Y')}{\Ex{Y'}{g(X,Y')}} } } \leq 0.
\end{equation}
Thus, by combining~\eqref{eq:com-pf-step-kl-1}-\eqref{eq:com-pf-step-kl-last},
\begin{equation}
\binrelentg{\Ex{X,Y}{g(X,Y)}}{\Ex{X,Y'}{g(X,Y')}}  \leq I(X;Y) .
\end{equation}
The desired result follows because~$\sup_\gamma \binrelentg{\cdot}{\cdot}=\binrelent{\cdot}{\cdot}$.

\end{proof}

\begin{lem}[Conditioning on independent random variables]\label{lem:cond-indep}
Consider the random variables~$X$,~$Y$ and~$Z$, where~$X$ and~$Z$ are independent.
Then,
\begin{equation}
I(X;Y) \leq I(X;Y\vert Z).
\end{equation}
\end{lem}
\begin{proof}
The result follows by using the independence of~$X$ and~$Z$ (which implies that~$I(X;Z)=0$), the chain rule for mutual information, and the non-negativity of mutual information as follows.
Note that
\begin{align}
I(X;Y,Z) = I(X;Z) + I(X;Y\vert Z) = I(X;Y\vert Z).
\end{align}
Alternatively,
\begin{align}
I(X;Y,Z) = I(X;Y) + I(X;Z\vert Y) \geq I(X;Y).
\end{align}
Thus,
\begin{align}
I(X;Y) \leq I(X;Y,Z) = I(X;Y\vert Z).
\end{align}
\end{proof}

\begin{lem}[Full-sample relaxation]\label{lem:full-sample}
Consider~$n$ independent random variables~$X=\{X_i\}_{i=1}^n$ and a random variable~$Y$.
Let~$\phi$ be a convex function.
Then,
\begin{equation}
\frac1n\sum_{i=1}^n \phi(I(X_i;Y) \leq \phi\lefto( \frac{I(X;Y)}{n} \righto).
\end{equation}
\end{lem}
\begin{proof}
Let~$X_{<i}$ denote~$\{X_1,\dots,X_{i-1} \}$. By the chain rule of mutual information,
\begin{align}
I(X;Y) = \sum_{i=1}^n I(X_i;Y\vert X_{<i}).
\end{align}
Due to the independence of the~$X_i$, Lemma~\ref{lem:cond-indep} implies that~$I(X_i;Y\vert X_{<i})\geq I(X_i;Y)$.
Combined with Jensen's inequality, this implies that
\begin{align}
\phi\lefto(\frac{I(X;Y)}{n} \righto)  \geq \phi\lefto(\sum_{i=1}^n\frac{I(X_i;Y)}{n}\righto) \geq \frac1n\sum_{i=1}^n \phi\lefto(I(X_i;Y) \righto).
\end{align}
\end{proof}

\begin{lem}[Sauer-Shelah lemma for the VC and Natarajan dimension]\label{lem:sauer-shelah}
Let~$g_\mathcal{F}(\cdot)$ denote the growth function of the function class~$\mathcal F$.
Specifically,~$g_\mathcal{F}(m)$ is the maximum number of different ways in which a data set of size~$m$ can be classified using functions from~$\mathcal F$.
For any function class~$\mathcal F$ with VC dimension~$\dVC$,
\begin{equation}\label{eq:lemma_sauer_shelah_for_vc_dim}
\gF(m) \leq \sum_{i=0}^{\dVC} \binom{m}{i} \leq \begin{cases}
                       \displaystyle 2^{\dVC+1}, &m<\dVC + 1 \\
                        \displaystyle \left(\frac{em}{\dVC} \right)^{\dVC}, &m\geq \dVC+1
                    \end{cases}
\end{equation}
More generally, for any function class~$\mathcal F$ with range~$\{0,\dots,N-1\}$ and Natarajan dimension~$d_N$,
\begin{equation}\label{eq:lemma_sauer_shelah_for_natarajan}
g_{\mathcal F}(m) \leq \sum_{i=0}^{d_N} \binom{m}{i} \binom{N}{2}^i\leq \begin{cases}
                      \displaystyle  N^{d_N+1}, &m<d_N + 1, \\
                     \displaystyle   \left(\binom{N}{2}\frac{em}{d_N} \right)^{d_N} , &m\geq d_N+1.
                    \end{cases}
\end{equation}
\end{lem}
\begin{proof}
The first inequality in~\eqref{eq:lemma_sauer_shelah_for_natarajan} follows from~\cite[Cor.~5]{haussler-95a} and the second follows from~\cite[Lemma~10]{guermeur-04a}.
The result in~\eqref{eq:lemma_sauer_shelah_for_vc_dim} follows by setting~$N=2$ in~\eqref{eq:lemma_sauer_shelah_for_natarajan}, for which the Natarajan dimension~$d_N$ coincides with the VC dimension~\cite[p.~222]{haussler-95a}.
\end{proof}

\subsection{Proofs for Section~\ref{sec:average_bounds}}\label{sec:proof-average}

\begin{proof}[Proof of Theorem~\ref{thm:two-step-sqrt-bound}]

We start by establishing a task-level generalization bound, i.e., a bound on~$\abs{ \avgmetapoploss-\avgunobstrainloss }$.
By Jensen's inequality, the convexity of~$\abs{\cdot}$ implies that
\begin{align}\label{eq:thm-1-pf-step-1}
 \abs{ \avgmetapoploss-\avgunobstrainloss } \leq \avgij \Ex{\metasupersample,\metasubsetchoice_i}{\abs{\Ex{\lambdaarg{i,\compl{\metasubsetchoice_i}}_j,\subsetchoiceinSi_j}{\lambdaijnSinSj - \lambdaijnSiSj}}} .
\end{align}
Let~$\subsetchoice'$ be an independent copy of~$\subsetchoice$.
By symmetry, we see that
\begin{equation}
    \Ex{\lambdaarg{i,\compl{\metasubsetchoice_i}}_j,\subsetchoicepinSij}{\lambdaijnSinSpj - \lambdaijnSiSpj} = 0.
\end{equation}
Using~\eqref{eq:lem_subgauss_com}, we can therefore bound the argument of the expectation in~\eqref{eq:thm-1-pf-step-1} as
\begin{equation}\label{eq:thm-1-pf-step-3}
\abs{\Ex{\lambdaarg{i,\compl{\metasubsetchoice_i}}_j,\subsetchoiceinSi_j}{\lambdaijnSinSj - \lambdaijnSiSj}}
    \leq \sqrt{2I^{\metasupersample,\metasubsetchoice_i}(\lambdaarg{i,\compl{\metasubsetchoice_i}}_j;\subsetchoiceinSi_j)}.
\end{equation}
Combining~\eqref{eq:thm-1-pf-step-1} and~\eqref{eq:thm-1-pf-step-3}, we obtain the following task-level generalization bound:
\begin{align}\label{eq:two-step-sqrt-pf-task}
 \abs{ \avgmetapoploss-\avgunobstrainloss } \leq \avgij \Ex{\metasupersample,\metasubsetchoice_i}{\sqrt{2I^{\metasupersample,\metasubsetchoice_i}(\lambdaarg{i,\compl{\metasubsetchoice_i}}_j;\subsetchoiceinSi_j)}}.
\end{align}
This is the first step of the two-step derivation.

Next, we establish an environment-level bound, i.e., a bound on~$\abs{ \avgunobstrainloss-\avgtrainloss }$.
Again, by Jensen's inequality, the convexity of~$\abs{\cdot}$ implies that
\begin{equation}
\abs{ \avgunobstrainloss-\avgtrainloss }\leq \avgij \Ex{\metasupersample,\subsetchoicei_j}{ \abs{\Ex{\lambdai_{j,\subsetchoicei_j},\metasubsetchoice_i}{ \lambdaijnSiSj - \lambdaijSiSj} } } .
\end{equation}
Symmetry implies that
\begin{equation}
    \Ex{\lambdai_{j,\subsetchoicei_j},\metasubsetchoice'_i}{ \lambdaijnSpiSj - \lambdaijSpiSj} = 0.
\end{equation}
We again bound the argument of the expectation, using~\eqref{eq:lem_subgauss_com}, as
\begin{equation}
 \abs{\Ex{\lambdai_{j,\subsetchoicei_j},\metasubsetchoice_i}{  \lambdaijnSiSj - \lambdaijSiSj}  }
\leq    {\sqrt{2I^{\metasupersample,\subsetchoicei_j}(\lambdai_{j,\subsetchoicei_j  };\metasubsetchoice_i)}}.
\end{equation}
Combining the two preceding inequalities, we obtain the following environment-level generalization guarantee, which is the second step:
\begin{equation}\label{eq:two-step-sqrt-pf-environment}
\abs{ \avgunobstrainloss-\avgtrainloss }\leq \avgij \Ex{\metasupersample,\subsetchoicei_j}{ {\sqrt{2I^{\metasupersample,\subsetchoicei_j}(\lambdai_{j,\subsetchoicei_j  };\metasubsetchoice_i)}} } .
\end{equation}
We conclude the proof by observing that~$\abs{\avgmetapoploss-\avgtrainloss}\leq \abs{\avgmetapoploss-\avgunobstrainloss}+\abs{\avgunobstrainloss-\avgtrainloss}$ by the triangle inequality, and by using~\eqref{eq:two-step-sqrt-pf-task} and~\eqref{eq:two-step-sqrt-pf-environment} to bound the two terms.

\end{proof}

\begin{proof}[Proof of Corollary~\ref{cor:two-step-sqrt-param}]

We begin with the first sum on the right-hand side of~\eqref{eq:thm-two-step-sqrt}.
By Jensen's inequality, 
\begin{equation}\label{eq:pf-of-cor-1-term-1-step-1}
\avgij \Ex{\metasupersample,\subsetchoicei_j}{ {\sqrt{2I^{\metasupersample,\subsetchoicei_j}(\lambdai_{j,\subsetchoicei_j  };\metasubsetchoice_i)}} }  \leq  {\sqrt{\avgij  2I(\lambdai_{j,\subsetchoicei_j  };\metasubsetchoice_i\vert\metasupersample,\subsetchoicei_j)}}  .
\end{equation}
By Lemma~\ref{lem:cond-indep} and the independence of~$\metasubsetchoice_i$ and~$(\subsetchoice,\randomness)$, we conclude that conditioning on~$\subsetchoice$ and~$\randomness$ does not increase the mutual information.
Hence,
\begin{equation}
{\sqrt{\avgij  2I(\lambdai_{j,\subsetchoicei_j  };\metasubsetchoice_i\vert\metasupersample,\subsetchoicei_j)}}  \leq {\sqrt{\avgij  2I(\lambdai_{j,\subsetchoicei_j  };\metasubsetchoice_i\vert\metasupersample,\subsetchoice,\randomness)}} .
\end{equation}
Since adding more random variables to the argument of the mutual information cannot decrease it, we have
\begin{equation}
{\sqrt{\avgij  2I(\lambdai_{j,\subsetchoicei_j  };\metasubsetchoice_i\vert\metasupersample,\subsetchoice,\randomness)}} \leq  {\sqrt{\avgij  2I(\lambda_{j,\subsetchoicei_j  };\metasubsetchoice_i\vert\metasupersample,\subsetchoice,\randomness)}} ,
\end{equation}
where~$\lambda_{j,\subsetchoicei_j}=\{ \lambdai_{j,\subsetchoicei_j} \}_{i=1}^{\metan}$.
Since the~$\metasubsetchoice_i$ are independent, it follows from Lemma~\ref{lem:full-sample} that
\begin{equation}
{\sqrt{\avgij  2I(\lambda_{j,\subsetchoicei_j  };\metasubsetchoice_i\vert\metasupersample,\subsetchoice,\randomness)}} \leq {\sqrt{\avgj  \frac{2I(\lambda_{j,\subsetchoicei_j  };\metasubsetchoice\vert\metasupersample,\subsetchoice,\randomness)}{\metan}}} .
\end{equation}
Now, note that given~$\metasupersample$,~$\subsetchoice$ and~$\randomness$, the losses~$\lambda_{j,\subsetchoicei_j}$ are a function of the output of the meta learner~$\hyperparam$.
Therefore,
\begin{equation}\label{eq:pf-of-cor-1-term-1-step-last}
{\sqrt{\avgj  \frac{2I(\lambda_{j,\subsetchoicei_j  };\metasubsetchoice\vert\metasupersample,\subsetchoice,\randomness)}{\metan}}} \leq \sqrt{  \frac{2I(\hyperparam ;\metasubsetchoice\vert\metasupersample,\subsetchoice,\randomness)}{\metan}} \leq \sqrt{  \frac{2I(\hyperparam ;\metasubsetchoice\vert\metasupersample,\subsetchoice)}{\metan}},
\end{equation}
where the last step follows from the independence of~$U$ and~$\randomness$.
By combining~\eqref{eq:pf-of-cor-1-term-1-step-1}-\eqref{eq:pf-of-cor-1-term-1-step-last}, we can bound the first sum in the right-hand side of~\eqref{eq:thm-two-step-sqrt} as
\begin{equation}\label{eq:pf-of-thm-1-param-sum-1}
\avgij\! \Ex{\metasupersample,\subsetchoicei_j \!}{ {\sqrt{2I^{\metasupersample,\subsetchoicei_j}(\lambdai_{j,\subsetchoicei_j  };\metasubsetchoice_i)}} }  \leq \sqrt{  \frac{2I(\hyperparam ;\metasubsetchoice\vert\metasupersample,\subsetchoice)}{\metan}}.
\end{equation}

For the second sum on the right-hand side of~\eqref{eq:thm-two-step-sqrt}, we again use Jensen's inequality to conclude that
\begin{equation}\label{eq:pf-of-cor-1-term-2-step-1}
\avgij \Ex{\metasupersample,\metasubsetchoice_i}{\sqrt{2I^{\metasupersample,\metasubsetchoice_i}(\lambdaarg{i,\compl{\metasubsetchoice_i}}_j;\subsetchoiceinSi_j)}} \leq\sqrt{\avgij2I(\lambdaarg{i,\compl{\metasubsetchoice_i}}_j;\subsetchoiceinSi_j\vert \metasupersample,\metasubsetchoice_i )}.
\end{equation}
Since adding more random variables does not decrease the mutual information,
\begin{equation}
\sqrt{\avgij2I(\lambdaarg{i,\compl{\metasubsetchoice_i}}_j;\subsetchoiceinSi_j\vert \metasupersample,\metasubsetchoice_i )} \leq \sqrt{\avgij2I(\lambdaarg{i,\compl{\metasubsetchoice_i}};\subsetchoiceinSi_j\vert \metasupersample,\metasubsetchoice_i )}.
\end{equation}
By Lemma~\ref{lem:cond-indep} and the independence of the~$\subsetchoiceinSi_j$,
\begin{equation}
\sqrt{\avgij2I(\lambdaarg{i,\compl{\metasubsetchoice_i}};\subsetchoiceinSi_j\vert \metasupersample,\metasubsetchoice_i )} \leq \sqrt{\avgi \frac{2I(\lambdaarg{i,\compl{\metasubsetchoice_i}};\subsetchoiceinSi\vert \metasupersample,\metasubsetchoice_i )}{\taskin}} .
\end{equation}
Since~$\randomness_i$,~$\metasubsetchoice_i$ and~$\subsetchoiceinSi$ have the same distribution for all~$i=1,\dots,\metan$,
\begin{equation}
\sqrt{\avgi \frac{2I(\lambdaarg{i,\compl{\metasubsetchoice_i}};\subsetchoiceinSi\vert \metasupersample,\metasubsetchoice_i )}{\taskin}}  = \sqrt{ \frac{2I(\lambdaarg{1,\compl{\metasubsetchoice_1}};\subsetchoicearg{1,\compl{\metasubsetchoice_1}}\vert \metasupersample,\metasubsetchoice_i )}{\taskin}} .
\end{equation}
Finally, given~$\metasupersample$ and~$\metasubsetchoice_i$,~$\lambdaarg{1,\compl{\metasubsetchoice_1}}$ is a function of~$\paramarg{1,\compl{\metasubsetchoice_1}}$. Hence,
\begin{equation}\label{eq:pf-of-cor-1-term-2-step-last}
\sqrt{ \frac{2I(\lambdaarg{1,\compl{\metasubsetchoice_1}};\subsetchoicearg{1,\compl{\metasubsetchoice_1}}\vert \metasupersample,\metasubsetchoice_i )}{\taskin}}    \leq \sqrt{ \frac{2I(\paramarg{1,\compl{\metasubsetchoice_1}};\subsetchoicearg{1,\compl{\metasubsetchoice_1}}\vert \metasupersample,\metasubsetchoice_i )}{\taskin}}  .
\end{equation}
By combining~\eqref{eq:pf-of-cor-1-term-2-step-1}-\eqref{eq:pf-of-cor-1-term-2-step-last}, we can bound the second term in the right-hand side of~\eqref{eq:thm-two-step-sqrt} as
\begin{equation}\label{eq:pf-of-thm-1-param-sum-2}
\avgij\! \Ex{\metasupersample,\metasubsetchoice_i\!}{\sqrt{2I^{\metasupersample,\metasubsetchoice_i}(\lambdaarg{i,\compl{\metasubsetchoice_i}}_j;\subsetchoicearg{i,\compl{\metasubsetchoice_i}}_j)}} \leq  \sqrt{ \frac{2I(\paramarg{1,\compl{\metasubsetchoice_1}};\subsetchoicearg{1,\compl{\metasubsetchoice_1}}\vert \metasupersample,\metasubsetchoice_i )}{\taskin}} .
\end{equation}
The result follows by combining~\eqref{eq:pf-of-thm-1-param-sum-1} and~\eqref{eq:pf-of-thm-1-param-sum-2}.
\end{proof}

\begin{proof}[Proof of Theorem~\ref{thm:one-step-sqrt-bound}]

By Jensen's inequality, we have
\begin{equation}\label{eq:one-step-pf-jensen}
\abs{\avgmetapoploss-\avgtrainloss} \leq \avgij \Ex{\metasupersample}{ \abs{\Ex{\lambdai_j,\metasubsetchoice_i,\subsetchoicei_j}{ \lambdaijnSinSj - \lambdaijSiSj}} }.
\end{equation}
Now, let~$\metasubsetchoice'$ and~$\subsetchoice'$ be independent copies of~$\metasubsetchoice$ and~$\subsetchoice$.
Note that
\begin{equation}
\Ex{\lambdai_j,\metasubsetchoice'_i,\subsetchoicepij}{ \lambdaijnSpinSpj - \lambdaijSpiSpj}=0. 
\end{equation}
We can therefore apply~\eqref{eq:lem_subgauss_com}, with~$X=\lambdai_j$ and~$Y$ being the pair of random variables~$(\metasubsetchoice'_i,\subsetchoicepij)$, to bound the argument of the expectation as
\begin{equation}\label{eq:one-step-pf-com}
\abs{\Ex{\lambdai_j,\metasubsetchoice_i,\subsetchoicei_j}{ \lambdaijnSinSj - \lambdaijSiSj}} \leq \sqrt{ 2 I^{\metasupersample}(\lambdai_j; \metasubsetchoice_i,\subsetchoicei_j)}.
\end{equation}
Combining~\eqref{eq:one-step-pf-jensen} and~\eqref{eq:one-step-pf-com}, we establish the desired result.

\end{proof}

\begin{proof}[Proof of Corollary~\ref{cor:one-step-sqrt-param}]
By Jensen's inequality,
\begin{equation}\label{eq:pf-of-cor-2-step-1}
\avgij \Ex{\metasupersample}{ \sqrt{ 2 I^{\metasupersample}(\lambdai_j; \metasubsetchoice_i,\subsetchoicei_j)} } \leq \sqrt{\avgij  2 I(\lambdai_j; \metasubsetchoice_i,\subsetchoicei_j\vert \metasupersample)}. 
\end{equation}
Since adding more random variables does not decrease the mutual information,
\begin{equation}
\sqrt{\avgij  2 I(\lambdai_j; \metasubsetchoice_i,\subsetchoicei_j\vert \metasupersample)}   \leq   \sqrt{\avgij  2 I(\lambdahat; \metasubsetchoice_i,\subsetchoicei_j\vert \metasupersample)}, 
\end{equation}
where~$\lambdahat=\{\lambdai_j \}_{j=1:\taskin}^{i=1:\metan}$.
By the independence of the~$\subsetchoicei_j$ for different~$j$, and~$(\metasubsetchoice_i,\subsetchoicei)$ for different~$i$,
\begin{equation}
\sqrt{\avgij  2 I(\lambdahat; \metasubsetchoice_i,\subsetchoicei_j\vert \metasupersample)}   \leq   \sqrt{\frac{  2 I(\lambdahat; \metasubsetchoice,\subsetchoice\vert \metasupersample)}{\taskin\metan}} . 
\end{equation}
Given~$\metasupersample$, the losses~$\lambdahat$ are a function of~$W=\{\paramarg{i}\}_{i=1}^{\metan}$. Thus,
\begin{equation}\label{eq:pf-of-cor-2-step-what}
\sqrt{\frac{  2 I(\lambdahat; \metasubsetchoice,\subsetchoice\vert \metasupersample)}{\taskin\metan}}   \leq   \sqrt{\frac{  2 I(W; \metasubsetchoice,\subsetchoice\vert \metasupersample)}{\taskin\metan}} . 
\end{equation}
Combining~\eqref{eq:pf-of-cor-2-step-1}-\eqref{eq:pf-of-cor-2-step-what}, we establish the first inequality in~\eqref{eq:cor-one-step-sqrt-param}.
Next, since adding random variables does not decrease mutual information,
\begin{align}
\sqrt{\frac{  2 I(W; \metasubsetchoice,\subsetchoice\vert \metasupersample)}{\taskin\metan}}   &\leq \sqrt{\frac{  2 I( \hyperparam,W; \metasubsetchoice,\subsetchoice\vert \metasupersample)}{\taskin\metan}} \\
&\leq  \sqrt{\frac{  2 I( \hyperparam; \metasubsetchoice,\subsetchoice\vert \metasupersample) + 2I( W; \metasubsetchoice,\subsetchoice\vert \metasupersample, \hyperparam) }{\taskin\metan}} 
\end{align}
where the second step follows from the chain rule.
Since the conditional distribution~$P_{W\metasubsetchoice\subsetchoice\vert\metasupersample\hyperparam}$ factorizes as~$P_{\metasubsetchoice\vert\metasupersample\hyperparam}\prod_{i=1}^{\metan}P_{\paramarg{i}\subsetchoicearg{i}\vert\metasupersample\hyperparam}$, we have that
\begin{align}
I( W; \metasubsetchoice,\subsetchoice\vert \metasupersample, \hyperparam) 
= \sumi I( \paramarg{i}; \subsetchoicei\vert \metasupersample, \hyperparam) .
\end{align}
Furthermore, since~$(\randomness_i,\metasubsetchoice_i,\subsetchoicei)$ are identically distributed for all~$i$,
\begin{equation}\label{eq:pf-of-cor-2-step-wiw1}
\sumi I( \paramarg{i}; \subsetchoicei\vert \metasupersample, \hyperparam)= \metan I( \paramarg{1}; \subsetchoicearg{1}\vert \metasupersample, \hyperparam).
\end{equation}
By combining~\eqref{eq:pf-of-cor-2-step-what}-\eqref{eq:pf-of-cor-2-step-wiw1}, we get
\begin{equation}\label{eq:pf-of-cor-2-state-cmi-chain}
\sqrt{\frac{  2 I(W; \metasubsetchoice,\subsetchoice\vert \metasupersample)}{\taskin\metan}} \leq \sqrt{\frac{  2 I( \hyperparam; \metasubsetchoice,\subsetchoice\vert \metasupersample)   +  2\metan I( \paramarg{1}; \subsetchoicearg{1}\vert \metasupersample, \hyperparam)    }{\taskin\metan}}.
\end{equation}
This establishes the second inequality in~\eqref{eq:cor-one-step-sqrt-param}.

Finally, by the chain rule,
\begin{align}
I( \hyperparam; \metasubsetchoice,\subsetchoice\vert \metasupersample) \leq I( \hyperparam; \metasubsetchoice,\subsetchoice\vert \metasupersample) + I(\hyperparam;\metasupersample) = I( \hyperparam;  \metasupersample,  \metasubsetchoice,\subsetchoice )=I(\hyperparam;\metatrainset).
\end{align}
Similarly,
\begin{equation}\label{eq:pf-of-cor-2-task-cmi-to-mi}
I( \paramarg{1}; \subsetchoicearg{1}\vert \metasupersample, \hyperparam)  \leq I( \paramarg{1}; \subsetchoicearg{1}\vert \metasupersample, \hyperparam)    +    I(\paramarg{1};\metasupersample\vert\hyperparam) = I( \paramarg{1}; \subsetchoicearg{1}, \metasupersample \vert \hyperparam) = I( \paramarg{1}; \supersamplearg{1}_{\subsetchoicearg{1}}\vert \hyperparam).
\end{equation}
By combining~\eqref{eq:pf-of-cor-2-state-cmi-chain}-\eqref{eq:pf-of-cor-2-task-cmi-to-mi}, we establish~\eqref{eq:cor-one-step-sqrt-param-cmi-to-mi}.
Thus, to summarize, we have shown that
\begin{align}\label{eq:cor-2-corrected-1}
\abs{\avgmetapoploss-\avgtrainloss}\leq  \sqrt{ \frac{2 I(W ; \metasubsetchoice,\subsetchoice \vert \metasupersample )}{\taskin\metan}} &\leq \sqrt{\frac{  2 I( \hyperparam; \metasubsetchoice,\subsetchoice\vert \metasupersample)   +  2\metan I( \paramarg{1}; \subsetchoicearg{1}\vert \metasupersample, \hyperparam)    }{\taskin\metan}} \\
    &\leq \sqrt{\frac{  2 I(\hyperparam;\metatrainset)   +  2\metan I( \paramarg{1}; \supersamplearg{1}_{\subsetchoicearg{1}}\vert \hyperparam)    }{\taskin\metan}}.\label{eq:cor-2-corrected-2}
\end{align}

\end{proof}

\begin{proof}[Proof of Theorem~\ref{thm:avg-kl-bounds}]

We begin by proving~\eqref{eq:thm-two-step-kl}.
First, we derive a task-level generalization bound.
By Jensen's inequality, we have
\begin{align}
 \binrelent{ \avgunobstrainloss }{ \frac{\avgunobstrainloss+\avgmetapoploss}{2} } \leq \avgij \Ex{\metasupersample,\metasubsetchoice_i}{\Ex{\lambdaarg{i,\compl{\metasubsetchoice_i}}_j,\subsetchoiceinSi_j}{\binrelent{\lambdaijnSiSj }{ \frac{\lambdaijnSinSj+\lambdaijnSiSj}{2}}}} .
\end{align}
Since~$\subsetchoiceinSi_j\in\{0,1\}$,~$\lambdaijnSinSj+\lambdaijnSiSj=\lambdaarg{i,\compl{\metasubsetchoice_i}}_{j,0}+\lambdaarg{i,\compl{\metasubsetchoice_i}}_{j,1}$ does not actually depend on~$\subsetchoiceinSi_j$.
Now, let~$\subsetchoice'$ be an independent copy of~$\subsetchoice$.
It follows that
\begin{equation}
\Ex{\lambdaarg{i,\compl{\metasubsetchoice_i}}_j,\subsetchoicepinSij}{ \lambdaijnSiSpj } =  \Ex{\lambdaarg{i,\compl{\metasubsetchoice_i}}_j}{\frac{\lambdaarg{i,\compl{\metasubsetchoice_i}}_{j,0}+\lambdaarg{i,\compl{\metasubsetchoice_i}}_{j,1}}{2}} . %
\end{equation}
We can thus use~\eqref{eq:lem_subgauss_com} to bound the argument of the expectation as 
\begin{equation}
\Ex{\lambdaarg{i,\compl{\metasubsetchoice_i}}_j,\subsetchoiceinSi_j}{\binrelent{\lambdaijnSiSj}{ \frac{\lambdaijnSinSj+\lambdaijnSiSj}{2} } } \leq I^{\metasupersample,\metasubsetchoice_i}(\lambdaarg{i,\compl{\metasubsetchoice_i}}_j;\subsetchoiceinSi_j).
\end{equation}
Combining the two inequalities, we obtain
\begin{equation}\label{eq:pf_task_level_kl_uninv}
\binrelent{ \avgunobstrainloss  }{ \frac{\avgunobstrainloss+\avgmetapoploss}{2} } \leq \avgij I(\lambdaarg{i,\compl{\metasubsetchoice_i}}_j;\subsetchoiceinSi_j\vert \metasupersample,\metasubsetchoice_i) .
\end{equation}
Recall that
\begin{equation}
\darg{2}(q,c)= \sup \bigg\{p\in[0,1]: \binrelent{q }{ \frac{q+p}{2} } \leq c\bigg\}.
\end{equation}
Using~$\darg{2}(\cdot)$ to invert~\eqref{eq:pf_task_level_kl_uninv}, we get
\begin{equation}
\avgmetapoploss\leq \darg{2}\lefto(\avgunobstrainloss, \avgij I(\lambdaarg{i,\compl{\metasubsetchoice_i}}_j;\subsetchoiceinSi_j\vert \metasupersample,\metasubsetchoice_i) \righto).
\end{equation}

Next, we perform similar steps at the environment level.
First, by Jensen's inequality,
\begin{equation}
\binrelent{ \avgtrainloss  }{   \frac{\avgtrainloss+\avgunobstrainloss}{2}\! }\!\leq\! \avgij \Ex{\metasupersample,\subsetchoicei_j\!}{ \Ex{\lambdai_{j,\subsetchoicei_j},\metasubsetchoice_i\!}{\! \binrelent{ \lambdaijSiSj  }{  \frac{\lambdaijSiSj+\lambdaijnSiSj}{2} }  }  }\! .\!\!\!
\end{equation}
Let~$\metasubsetchoice'$ be an independent copy of~$\metasubsetchoice$.
By a similar argument as in the proof of the task-level bound,
\begin{equation}
    \Ex{\lambdai_{j,\subsetchoicei_j},\metasubsetchoice'_i}{ \lambdaijSpiSj} =  \Ex{\lambdai_{j,\subsetchoicei_j}}{ \frac{  \lambdaarg{i,0}_{j,\subsetchoicei_j}  +  \lambdaarg{i,1}_{j,\subsetchoicei_j}  }{2}} .
\end{equation}
We can therefore again bound the argument of the expectation with~\eqref{eq:lem_kl_com} to obtain
\begin{equation}
\Ex{\lambdai_{j,\subsetchoicei_j},\metasubsetchoice_i\!}{\! \binrelent{ \lambdaijSiSj  }{ \frac{\lambdaijSiSj+\lambdaijnSiSj}{2}  } } 
\leq    I^{\metasupersample,\subsetchoicei_j}(\lambdai_{j,\subsetchoicei_j  };\metasubsetchoice_i).
\end{equation}
By combining the two inequalities, we find that
\begin{equation}
\binrelent{\avgtrainloss  }{   \frac{\avgtrainloss+\avgunobstrainloss}{2} }  \leq \avgij I(\lambdai_{j,\subsetchoicei_j  };\metasubsetchoice_i\vert \metasupersample,\subsetchoicei_j)
\end{equation}
which, through the use of~$\darg{2}(\cdot)$, implies that
\begin{equation}
\avgunobstrainloss \leq \darg{2}\lefto( \avgtrainloss, \avgij I(\lambdai_{j,\subsetchoicei_j  };\metasubsetchoice_i\vert \metasupersample,\subsetchoicei_j) \righto).
\end{equation}
To complete the proof, we use the following observation.
Assume that~$\avgmetapoploss \leq  B(\avgunobstrainloss)$, where~$ B(\cdot)$ is a non-decreasing function.
Then, if~$\avgunobstrainloss\leq \widehat B(\avgtrainloss)$, we have $\avgmetapoploss \leq  B(\widehat B(\avgtrainloss))$.
To apply this observation, we note that~$\darg{2}(\cdot,c)$ is non-decreasing for~$c>0$.
Chaining the two bounds, we obtain
\begin{align}
\avgmetapoploss &\leq  \darg{2}\lefto(\avgunobstrainloss, \avgij I(\lambdaarg{i,\compl{\metasubsetchoice_i}}_j;\subsetchoiceinSi_j\vert \metasupersample,\metasubsetchoice_i) \righto) \\
&\leq \darg{2}\lefto(\darg{2}\lefto( \avgtrainloss, \avgij I(\lambdai_{j,\subsetchoicei_j  };\metasubsetchoice_i\vert \metasupersample,\subsetchoicei_j) \righto), \avgij I(\lambdaarg{i,\compl{\metasubsetchoice_i}}_j;\subsetchoiceinSi_j\vert \metasupersample,\metasubsetchoice_i) \righto).
\end{align}
This establishes~\eqref{eq:thm-two-step-kl}.

Next, we turn to~\eqref{eq:thm-one-step-kl}.
By Jensen's inequality, we have
\begin{multline}\label{eq:one-step-pf-jensen-kl}
\binrelent{  \avgtrainloss  }{ \frac{\avgtrainloss+\avgtestloss+\avgunobstrainloss+\avgmetapoploss}{4} } \\\!\leq\! \avgij \Ex{\metasupersample\!}{ \Ex{\lambdai_j,\metasubsetchoice_i,\subsetchoicei_j\!}{  \binrelent{\!\lambdaijSiSj  }{  \frac{\lambdaijSiSj\! +\! \lambdaijSinSj \!+\! \lambdaijnSiSj \!+\!  \lambdaijnSinSj}{4}   }   }} \!.\!\!\!
\end{multline}
Let~$\metasubsetchoice'$ and~$\subsetchoice'$ be independent copies of~$\metasubsetchoice$ and~$\subsetchoice$ respectively.
We note that
\begin{equation}
\Ex{\lambdai_j,\metasubsetchoice'_i,\subsetchoicepij\!}{  \lambdaijSiSj} = \Ex{\lambdai_j}{ \frac{\lambdaarg{i,0}_{j,0}\! +\! \lambdaarg{i,0}_{j,1} \!+\! \lambdaarg{i,1}_{j,0} \!+\!  \lambdaarg{i,1}_{j,1}}{4}    }  .
\end{equation}
This means that we can apply~\eqref{eq:lem_kl_com} to the argument of the expectation to get
\begin{equation}\label{eq:one-step-pf-com-kl}
\Ex{\lambdai_j,\metasubsetchoice_i,\subsetchoicei_j\!\!}{  \binrelent{  \!\lambdaijSiSj  }{ \frac{\lambdaijSiSj\! +\! \lambdaijSinSj \!+\! \lambdaijnSiSj \!+\!  \lambdaijnSinSj}{4} }\!   }\! \leq \! I^{\metasupersample}(\lambdai_j; \metasubsetchoice_i,\subsetchoicei_j ).\!\!
\end{equation}
The result in~\eqref{eq:thm-one-step-kl} now follows by combining~\eqref{eq:one-step-pf-jensen-kl} and~\eqref{eq:one-step-pf-com-kl}.

\end{proof}

\begin{proof}[Proof of Corollary~\ref{cor:interpolating-one-step-kl}]

Since~$\lim_{x\rightarrow0^+} x\log x=0$, we use the convention that~$0\log 0=0$.
From the definition of~$\binrelent{\cdot}{\cdot}$, we thus get
\begin{equation}
\binrelent{0}{\frac p4} = \log\frac{1}{1-\frac p4}
\end{equation}
from which it follows that
\begin{equation}\label{eq:d4-inv-eq-pf}
\darg{4}(0,c) =4-4e^{-c}.
\end{equation}
By the non-negativity of the loss function,~$\avgmetapoploss\leq\avgmetapoploss+\avgtestloss+\avgunobstrainloss$.
Thus, combining~\eqref{eq:thm-one-step-kl} with~\eqref{eq:d4-inv-eq-pf}, we obtain the result in~\eqref{eq:cor-interp-one-step-kl}.

\end{proof}

\subsection{Proofs for Section~\ref{sec:high-probability}}\label{sec:proof-high-probability}

To derive the simplified result stated in Theorem~\ref{thm:high-prob-two-step-sqrt}, we assume that the meta learner and base learner are invariant to the order of the data samples.
However, this assumption is only necessary to simplify the expression, and a similar bound holds more generally without this assumption.
Therefore, we first state and prove this more general result in Theorem~\ref{thm:high-prob-two-step-sqrt-unsimplified}.
Then, we describe how to simplify the result to obtain Theorem~\ref{thm:high-prob-two-step-sqrt}.
Later, when proving Corollary~\ref{cor:rep-learning-high-probability}, we will use the more general Theorem~\ref{thm:high-prob-two-step-sqrt-unsimplified} as the basis of the derivation.

\begin{thm}\label{thm:high-prob-two-step-sqrt-unsimplified}
Consider the setting introduced in Section~\ref{sec:notation}.
For each~$j$, let~$Q_{j,\subsetchoice_j}$ denote the conditional distribution of~$\lambda_{j,\subsetchoice_j}$ given~$(\metasupersample, \metasubsetchoice,\subsetchoice )$, and let~$P_{j,\subsetchoice_j}$ denote~$\Ex{\metasubsetchoice}{Q_{j,\subsetchoice_j}}$.
Furthermore, let~$Q^{i,\compl{\metasubsetchoice_i}}$ denote the conditional distribution of~$\lambdaarg{i,\compl{\metasubsetchoice_i}}$ given~$\pacbtraindata$, and let~$P^{i,\compl{\metasubsetchoice_i}}$ denote~$\Ex{\subsetchoicearg{\compl{\metasubsetchoice}}}{Q^{i,\compl{\metasubsetchoice_i}}}$.
Then, with probability at least~$1-\delta$ under the draw of~$\pacbtraindata$,
\begin{multline}\label{eq:app-thm-two-step-unsimplified}
\abs{\avgmetapoploss - \avgtrainloss}  \leq  \sqrt{ \frac{\avgj 2  \relent{Q_{j,\subsetchoice_j}}{P_{j,\subsetchoice_j}}   +   2\log(\frac{2\taskin\sqrt {\metan}}{\delta})}{\metan-1} }  \\
+  \sqrt{ \frac{  \sumi 2\relent{Q^{i,\compl{\metasubsetchoice_i}}}{P^{i,\compl{\metasubsetchoice_i}}}  +  2\log( \frac{2\sqrt {\taskin\metan}}{\delta} ) }{\taskin\metan-1} }.   
\end{multline}
\end{thm}

\begin{proof}[Proof of Theorem~\ref{thm:high-prob-two-step-sqrt-unsimplified}]

First, let~$\pacbavgunobstrainloss$ denote the training loss on unobserved tasks,
\begin{align}
\pacbavgunobstrainloss &= \avgij \Ex{\metarandomness,\randomness_i}{\ell(\baselearner(\supersamplearg{i,\compl{\metasubsetchoice_i}}_{\subsetchoicei},\randomness_i,\metalearner(\metatrainset,\metarandomness)) , \supersamplearg{i,\compl{\metasubsetchoice_i}}_{j,{\compl{\subsetchoicei_j}} } )}.
\end{align}
We begin by establishing an environment-level bound.
Let~$\lambda_{j,\subsetchoice_j}$ be distributed according to~$Q_{j,\subsetchoice_j}$.
By Jensen's inequality,
\begin{align}\label{eq:pf-of-tail-bound-environ-step-1}
\frac{\metan-1}{2}\lefto(\pacbavgunobstrainloss-\pacbavgtrainloss\righto)^2 \leq \avgj
\Ex{\lambda_{j,\subsetchoice_j}}{  \frac{\metan-1}{2}\lefto(\avgi \lambdaijnSiSj-\lambdaijSiSj   \righto)^2 }.
\end{align}
Now, let~$\lambda'_{j,\subsetchoice_j}$ be distributed according to~$P_{j,\subsetchoice_j}$.
By Donsker-Varadhan's variational representation of the KL divergence,
\begin{multline}\label{eq:pf-of-tail-bound-environ-step-dv}
\avgj
\Ex{\lambda_{j,\subsetchoice_j}}{  \frac{\metan-1}{2}\lefto(\avgi \lambdaijnSiSj-\lambdaijSiSj   \righto)^2 } \\
\leq \avgj\lefto( \relent{Q_{j,\subsetchoice_j}}{P_{j,\subsetchoice_j}}  +  \log \Ex{ \lambda'_{j,\subsetchoice_j} }{ \exp\lefto( \frac{\metan-1}{2}\lefto(\avgi \lambdapijnSiSj-\lambdapijSiSj   \righto)^2 \righto) }\righto).
\end{multline}
For each~$j$, Markov's inequality implies that, with probability at least~$1-\delta$ under the draw of~$\pacbtraindata$,
\begin{multline}\label{eq:pf-environ-tail-bound-shat-independent-per-i}
 \relent{Q_{j,\subsetchoice_j}}{P_{j,\subsetchoice_j}}  +  \log \Ex{ \lambda'_{j,\subsetchoice_j} }{ \exp\lefto( \frac{\metan-1}{2}\lefto(\avgi \lambdapijnSiSj-\lambdapijSiSj   \righto)^2 \righto) }\\
\leq \relent{Q_{j,\subsetchoice_j}}{P_{j,\subsetchoice_j}}  +  \log  \Ex{ \lambda'_{j,\subsetchoice_j},\pacbtraindatan }{ \frac1\delta \exp\lefto( \frac{\metan-1}{2}\lefto(\avgi \lambdapijnSiSj-\lambdapijSiSj   \righto)^2 \righto) }  .
\end{multline}
By the union bound, this implies that, with~$\delta\rightarrow \delta/\taskin$,~\eqref{eq:pf-environ-tail-bound-shat-independent-per-i} holds for all~$j$ simultaneously with probability at least~$1-\delta$.
Thus, with probability at least~$1-\delta$ under the draw of~$\pacbtraindata$,
\begin{multline}\label{eq:pf-environ-tail-bound-shat-independent}
\avgj\lefto( \relent{Q_{j,\subsetchoice_j}}{P_{j,\subsetchoice_j}}  +  \log \Ex{ \lambda'_{j,\subsetchoice_j} }{ \exp\lefto( \frac{\metan-1}{2}\lefto(\avgi \lambdapijnSiSj-\lambdapijSiSj   \righto)^2 \righto) }\righto)\\
\leq\! \avgj\lefto(\! \relent{Q_{j,\subsetchoice_j\!}}{\!P_{j,\subsetchoice_j}\!}  \!+\!  \log  \Ex{ \lambda'_{j,\subsetchoice_j},\pacbtraindatan\! }{\! \frac{\taskin}\delta\! \exp\lefto(\! \frac{\metan\!-\!1}{2}\lefto(\avgi \lambdapijnSiSj\!-\!\lambdapijSiSj   \righto)^{\! 2} \righto) \!}\!\righto)\!.
\end{multline}
Note that, on the right-hand side of~\eqref{eq:pf-environ-tail-bound-shat-independent},~$\metasubsetchoice$ is independent from~$(\lambda'_{j,\subsetchoice_j},\metasupersample,\subsetchoice)$.
Furthermore, for each~$(i,j)$,~$\lambdapijnSiSj-\lambdapijSiSj$ is bounded to~$[-1,1]$ and~$\Ex{\metasubsetchoice_i}{\lambdapijnSiSj-\lambdapijSiSj }=0$.
This implies that~$\avgi\lambdapijnSiSj-\lambdapijSiSj$ is a~$1/\sqrt{\metan}$-sub-Gaussian random variable, from which it follows that~\cite[Thm.~2.6.(IV)]{wainwright19-a}
\begin{equation}\label{eq:pf-of-tail-environ-conc}
\log  \Ex{ \lambda'_{j,\subsetchoice_j},\pacbtraindatan }{ \exp\lefto( \frac{\metan-1}{2}\lefto(\avgi \lambdapijnSiSj-\lambdapijSiSj   \righto)^2 \righto) } \leq \log(\sqrt {\metan}).
\end{equation}
By substituting~\eqref{eq:pf-of-tail-environ-conc} into~\eqref{eq:pf-environ-tail-bound-shat-independent}, we obtain
\begin{multline}\label{eq:pf-of-tail-bound-environ-step-after-conc}
\avgj\lefto( \relent{Q_{j,\subsetchoice_j}}{P_{j,\subsetchoice_j}}  +  \log \Ex{ \lambda'_{j,\subsetchoice_j} }{ \exp\lefto( \frac{\metan-1}{2}\lefto(\avgi \lambdapijnSiSj-\lambdapijSiSj   \righto)^2 \righto) }\righto)\\
\leq \avgj \relent{Q_{j,\subsetchoice_j}}{P_{j,\subsetchoice_j}}  +  \log\lefto(\frac{\taskin\sqrt{\metan}}{\delta}\righto).
\end{multline}
By combining~\eqref{eq:pf-of-tail-bound-environ-step-1}-\eqref{eq:pf-of-tail-bound-environ-step-after-conc}, we get, after some arithmetic,
\begin{equation}\label{eq:pf-of-tail-environ-level-bound}
\abs{\pacbavgunobstrainloss-\pacbavgtrainloss} \leq \sqrt{ \frac{\avgj 2  \relent{Q_{j,\subsetchoice_j}}{P_{j,\subsetchoice_j}}  +  2\log\lefto(\frac{\taskin\sqrt {\metan}}{\delta}\righto)}{\metan-1} }.
\end{equation}

We now turn to the task level.
Let~$Q^{\compl{\metasubsetchoice}}$ denote the conditional distribution of~$\lambdaarg{\compl{\metasubsetchoice}}$ given~$\pacbtraindata$, and let~$P^{\compl{\metasubsetchoice}}$ denote~$\Ex{\subsetchoicearg{\compl{\metasubsetchoice}}}{Q^{\compl{\metasubsetchoice}}}$.
Let~$\lambdaarg{\compl{\metasubsetchoice}}$ be distributed according to~$Q^{\compl{\metasubsetchoice}}$.
By Jensen's inequality,
\begin{align}\label{eq:pf-of-tail-bound-task-step-1}
\frac{\taskin\metan-1}{2}\lefto(\pacbavgmetapoploss\!-\!\pacbavgunobstrainloss\righto)^2 \!\leq\! 
\Ex{\lambdaarg{\compl{\metasubsetchoice}}}{  \frac{\taskin\metan-1}{2}\lefto(\avgij \lambdaijnSinSj-\lambdaijnSiSj   \righto)^2 }.\!\!
\end{align}
Now, let~${\lambda'}^{\compl{\metasubsetchoice}}$ be distributed according to~$P^{\compl{\metasubsetchoice}}$.
By Donsker-Varadhan's variational representation of the KL divergence,
\begin{multline}\label{eq:pf-of-tail-bound-task-step-dv}
\Ex{\lambdaarg{\compl{\metasubsetchoice}}}{  \frac{\taskin\metan-1}{2}\lefto(\avgij \lambdaijnSinSj-\lambdaijnSiSj   \righto)^2 } \\
\leq   \relent{Q^{\compl{\metasubsetchoice}}}{P^{\compl{\metasubsetchoice}}} + \log \Ex{ {\lambda'}^{\compl{\metasubsetchoice}} }{ \exp\lefto( \frac{\taskin\metan-1}{2} \lefto( \avgij \lambdapijnSinSj-\lambdapijnSiSj \righto)^2  \righto) } .
\end{multline}
By Markov's inequality, we conclude that with probability at least~$1-\delta$ under~$\pacbtraindata$,
\begin{multline}\label{eq:pf-task-tail-bound-s-independent}
 \relent{Q^{\compl{\metasubsetchoice}}}{P^{\compl{\metasubsetchoice}}} + \log \Ex{ {\lambda'}^{\compl{\metasubsetchoice}} }{ \exp\lefto( \frac{\taskin\metan-1}{2} \lefto(\avgij \lambdapijnSinSj-\lambdapijnSiSj \righto)^2  \righto) }  \\
\leq \!  \relent{Q^{\compl{\metasubsetchoice}}}{P^{\compl{\metasubsetchoice}}} \!+\! \log \Ex{ {\lambda'}^{\compl{\metasubsetchoice}},\pacbtraindatan }{\frac1\delta \exp\lefto( \frac{\taskin\metan-1}{2} \lefto(\avgij \lambdapijnSinSj\!-\!\lambdapijnSiSj \!\righto)^2 \! \righto)\! } .
\end{multline}
Note that~$\subsetchoice^{\compl{\metasubsetchoice}}$ is independent from~$({\lambda'}^{\compl{\metasubsetchoice}},\metasupersample,\metasubsetchoice,\subsetchoice^{\metasubsetchoice})$, that for each~$(i,j)$,~$\lambdapijnSinSj-\lambdapijnSiSj$ is bounded to~$[-1,1]$, and that~$\Ex{ \subsetchoicearg{i,\compl{\metasubsetchoice_i}}_j }{\lambdapijnSinSj-\lambdapijnSiSj}=0$.
Thus, it follows that~\cite[Thm.~2.6.(IV)]{wainwright19-a}
\begin{equation}\label{eq:pf-of-tail-task-conc}
\log \Ex{ {\lambda'}^{\compl{\metasubsetchoice}},\pacbtraindatan }{ \exp\lefto( \frac{\taskin\metan-1}{2} \lefto(\avgij \lambdapijnSinSj-\lambdapijnSiSj \righto)^2  \righto) }  \leq \log\lefto(\sqrt {\taskin\metan}\righto).
\end{equation}
By substituting~\eqref{eq:pf-of-tail-task-conc} into~\eqref{eq:pf-task-tail-bound-s-independent}, we obtain
\begin{multline}\label{eq:pf-task-tail-bound-step-after-conc}
 \relent{Q^{\compl{\metasubsetchoice}}}{P^{\compl{\metasubsetchoice}}} + \log \Ex{ {\lambda'}^{\compl{\metasubsetchoice}} }{ \exp\lefto( \frac{\taskin\metan-1}{2} \lefto(\avgij \lambdapijnSinSj-\lambdapijnSiSj \righto)^2  \righto) }  \\
\leq   \relent{Q^{\compl{\metasubsetchoice}}}{P^{\compl{\metasubsetchoice}}} + \log\lefto(\frac{\sqrt{\taskin\metan}}{\delta} \righto) .
\end{multline}
Since~$\lambdaarg{i,\compl{\metasubsetchoice_i}}$ and~$\lambdaarg{i',\compl{\metasubsetchoice_{i'}}}$ are losses on separate, unobserved tasks, they are dependent only through~$U$.
Therefore, they are conditionally independent given~$(\metasupersample,\metasubsetchoice,\subsetchoice^{\metasubsetchoice})$.
By the chain rule for the KL divergence, it follows that
\begin{equation}\label{eq:pf-task-tail-bound-step-kl-decomp}
\relent{Q^{\compl{\metasubsetchoice}}}{P^{\compl{\metasubsetchoice}}} = \sumi \relent{Q^{i,\compl{\metasubsetchoice_i}}}{P^{i,\compl{\metasubsetchoice_i}}}.
\end{equation}

By combining~\eqref{eq:pf-of-tail-bound-task-step-1}-\eqref{eq:pf-task-tail-bound-step-kl-decomp}, we get, after some arithmetic, that with probability at least~$1-\delta$ under~$\pacbtraindata$,
\begin{equation}\label{eq:pf-of-tail-task-level-bound}
\abs{\avgmetapoploss-\avgunobstrainloss} \leq \sqrt{ \frac{  \sumi 2\relent{Q^{i,\compl{\metasubsetchoice_i}}}{P^{i,\compl{\metasubsetchoice_i}}} + 2\log\lefto( \frac{\sqrt {\taskin\metan}}{\delta} \righto) }{\taskin\metan-1} }.
\end{equation}
By the triangle inequality,~$\abs{\avgmetapoploss-\avgtrainloss}\leq \abs{\avgmetapoploss-\avgunobstrainloss} + \abs{\avgunobstrainloss-\avgtrainloss} $.
By the union bound,~\eqref{eq:pf-of-tail-environ-level-bound} and~\eqref{eq:pf-of-tail-task-level-bound} hold simultaneously with probability at least~$1-2\delta$ under~$\pacbtraindata$.
Therefore, with~$\delta\rightarrow\delta/2$, they hold simultaneously with probability at least~$1-\delta$.
Thus, with probability at least~$1-\delta$ under the draw of~$\pacbtraindata$,
\begin{multline}\label{eq:pf-of-tail-combined-level-bound}
\abs{\avgmetapoploss\!-\!\avgtrainloss} \!\leq\! \sqrt{ \frac{\avgj 2  \relent{Q_{j,\subsetchoice_j}}{P_{j,\subsetchoice_j}}  \!+\!  2\log\lefto(\frac{2\taskin\sqrt {\metan}}{\delta}\righto)}{\metan-1} } \!\\
+ \!\sqrt{ \frac{  \sumi 2\relent{Q^{i,\compl{\metasubsetchoice_i}}}{P^{i,\compl{\metasubsetchoice_i}}} \!+\! 2\log\lefto( \frac{2\sqrt {\taskin\metan}}{\delta} \righto) }{\taskin\metan-1} }.\!\!\!
\end{multline}
\end{proof}

Having established Theorem~\ref{thm:high-prob-two-step-sqrt-unsimplified}, we now show how to use it to derive Theorem~\ref{thm:high-prob-two-step-sqrt} under the assumption that the meta learner and base learner are invariant to the order of the samples.

\begin{proof}[Proof of Theorem~\ref{thm:high-prob-two-step-sqrt}]
By the assumptions that the meta learner and base learner are invariant to the sample order and task index, we can reorder the data set so that~$\max_j\relent{Q_{j,\subsetchoice_j}}{P_{j,\subsetchoice_j}}=\relent{Q_{1,\subsetchoice_1}}{P_{1,\subsetchoice_1}}$.
With this,~$\avgj  \relent{Q_{j,\subsetchoice_j}}{P_{j,\subsetchoice_j}} \leq \relent{Q_{1,\subsetchoice_1}}{P_{1,\subsetchoice_1}}$.
Similarly, we can reorder the data set so that~$\max_i\relent{Q^{i,\compl{\metasubsetchoice_i}}}{P^{i,\compl{\metasubsetchoice_i}}}=\relent{Q_\lambda^{1,\compl{\metasubsetchoice_1}}}{P_\lambda^{1,\compl{\metasubsetchoice_1}}}$, implying that~$\sumi \relent{Q^{i,\compl{\metasubsetchoice_i}}}{P^{i,\compl{\metasubsetchoice_i}}}\leq\metan \relent{Q_\lambda^{1,\compl{\metasubsetchoice_1}}}{P_\lambda^{1,\compl{\metasubsetchoice_1}}}$.
To obtain the final result, we note that for~$\taskin,\metan\geq 2$, we have~$1/(\metan\!-\!1)\leq 2/\metan$, $\log 2\leq \log(\taskin\sqrt {\metan}/\delta)$,~$1/(\taskin\!-\!1)\leq 2/\taskin$,~$\log 2\leq \log(\sqrt {\taskin\metan}/\delta)$, and~$\log(\sqrt{\taskin\metan})/\taskin\metan\leq \log(\sqrt{\taskin})/\taskin$.
Thus, we get the final result
\begin{equation}\label{eq:pf-of-tail-combined-level-bound-simplified}
\abs{\avgmetapoploss\!-\!\avgtrainloss} \!\leq\! 2\sqrt2\sqrt{\! \frac{  \relent{Q_{1,\subsetchoice_1}\!}{\!P_{1,\subsetchoice_1}}  \!+\!  \log\lefto(\frac{\taskin\sqrt {\metan}}{\delta}\righto)}{\metan} } \!+\! 2\sqrt2\sqrt{\! \frac{  \relent{Q^{1,\compl{\metasubsetchoice_1}}\!}{\!P^{1,\compl{\metasubsetchoice_1}}} \!+\! \log\lefto( \frac{\sqrt \taskin}{\delta} \righto) }{\taskin} }.
\end{equation}
Thus, the bound holds with~$C_1=C_2=2\sqrt 2$.

\end{proof}

While the simplifying assumption of invariance to the order of samples leads to a simpler result, it does not hold for all learning algorithms.
Therefore, we will use the more general form given in~\eqref{eq:pf-of-tail-combined-level-bound} as the basis of Corollary~\ref{cor:rep-learning-high-probability}.

Note that the first term of~\eqref{eq:pf-of-tail-combined-level-bound-simplified} diverges as~$\taskin\rightarrow \infty$.
This counter-intuitive behavior, which requires both~$\taskin$ and~$\metan$ to be large for the bound to be nonvacuous, is common in PAC-Bayesian bounds for meta learning~\cite{amit-18a}, and is seemingly an effect of the two-step approach.
It is possible to obtain a different bound where this dependence is not explicit, similar to~\cite{pentina-14a,guan-21a}, by simply not applying Jensen's inequality to the average over~$j$ in~\eqref{eq:pf-of-tail-bound-environ-step-1} nor the average over~$i$ in~\eqref{eq:pf-of-tail-bound-task-step-1}. However, the resulting environment-level KL divergence is different.
In particular, if we were to use this alternative bound to derive minimax bounds in Section~\ref{sec:expressiveness}, the logarithmic dependence on~$\taskin$ would be embedded in this KL divergence.
In the proof of Corollary~\ref{cor:rep-learning-high-probability}, we point out where this difference would come into play.

We present the alternative bound in the following remark.
\begin{rem}\label{rem:alternative-high-prob-two-step-sqrt}
Let~$Q_{\subsetchoice}$ denote the conditional distribution of~$\lambda_{\subsetchoice}$ given~$\pacbtraindata$ and let~$P_{\subsetchoice}=\Ex{\metasubsetchoice}{Q_{\subsetchoice}}$.
Then,
\begin{equation}\label{eq:pf-of-tail-combined-level-bound-simplified-alternative-version}
\abs{\avgmetapoploss\!-\!\avgtrainloss} \!\leq\! 2\sqrt2\sqrt{ \frac{  \relent{Q_{\subsetchoice}}{P_{\subsetchoice}}  \!+\!  \log\lefto(\frac{\sqrt {\metan}}{\delta}\righto)}{\metan} } \!+ \!2\sqrt2\sqrt{ \frac{  \relent{Q^{1,\compl{\metasubsetchoice_1}}}{P^{1,\compl{\metasubsetchoice_1}}}  \!+\! \log\lefto( \frac{\sqrt \taskin}{\delta} \righto) }{\taskin} }.\!\!\!
\end{equation}
\end{rem}
\begin{proof}
The proof follows that of Theorem~\ref{thm:high-prob-two-step-sqrt}, so we only detail the differences:
In~\eqref{eq:pf-of-tail-bound-environ-step-1}, the average over~$j$ is not moved outside the square when using Jensen's inequality;
Donsker-Varadhan is now used to perform a change of measure from~$Q_{\subsetchoice}$ to~$P_{\subsetchoice}$ in~\eqref{eq:pf-of-tail-bound-environ-step-dv};
and the union bound in~\eqref{eq:pf-environ-tail-bound-shat-independent} is no longer needed.
\end{proof}

We now turn to Theorem~\ref{thm:high-prob-one-step-sqrt}.

\begin{proof}[Proof of Theorem~\ref{thm:high-prob-one-step-sqrt}]
Recall that~$Q$ denotes the conditional distribution of~$\lambdahat$ given~$(\metasupersample,\metasubsetchoice,\subsetchoice)$, and that~$P$ denotes~$\Ex{\metasubsetchoice,\subsetchoice}{Q}$.
Let~$\lambda$ be distributed according to~$Q$.
By Jensen's inequality,
\begin{align}\label{eq:pf-of-tail-bound-one-step-1}
\frac{\taskin\metan-1}{2}\lefto(\pacbavgmetapoploss\!-\!\pacbavgtrainloss\righto)^2 \!\leq\! 
\Ex{ {\lambda} }{  \frac{\taskin\metan-1}{2}\lefto(\avgij \lambdaijnSinSj-\lambdaijSiSj   \righto)^2 }.
\end{align}
Next, let~$\lambda'$ be distributed according to~$P$.
By Donsker-Varadhan's variational representation of the KL divergence,
\begin{multline}
\Ex{ {\lambda} }{  \frac{\taskin\metan-1}{2}\lefto(\avgij \lambdaijnSinSj-\lambdaijSiSj   \righto)^2 } \\
\leq \relent{Q}{P} + \log \Ex{\lambda'}{   \frac{\taskin\metan-1}{2}\lefto(\avgij \lambdapijnSinSj-\lambdapijSiSj   \righto)^2 }.
\end{multline}
By Markov's inequality, we conclude that with probability at least~$1-\delta$ under the draw of~$\pacbtraindata$,
\begin{multline}\label{eq:pf-one-step-tail-bound-ss-independent}
\relent{Q}{P} + \log \Ex{\lambda'}{   \exp\lefto(\frac{\taskin\metan-1}{2}\lefto(\avgij \lambdapijnSinSj-\lambdapijSiSj   \righto)^2  \righto) }  \\
\leq  \relent{Q}{P} + \log \Ex{\lambda',\pacbtraindatan}{ \frac1\delta  \exp\lefto( \frac{\taskin\metan-1}{2}\lefto(\avgij \lambdapijnSinSj-\lambdapijSiSj   \righto)^2  \righto)  }.
\end{multline}
Now, note that~$(\metasubsetchoice,\subsetchoice)$ are independent from~$\lambda',\metasupersample$.
Furthermore,~$\lambdapijnSinSj-\lambdapijSiSj$ is bounded to~$[-1,1]$, and~$\Ex{\metasubsetchoice_i,\subsetchoicei_j}{\lambdapijnSinSj-\lambdapijSiSj  }=0$.
Thus,~$\avgij \lambdapijnSinSj-\lambdapijSiSj $ is a~$1/\sqrt{\taskin\metan}$-sub-Gaussian random variable, from which it follows that~\cite[Thm.~2.6.(IV)]{wainwright19-a}
\begin{equation}\label{eq:pf-of-tail-one-step-conc}
\log \Ex{\lambda',\pacbtraindatan}{ \exp\lefto( \frac{\taskin\metan-1}{2}\lefto(\avgij \lambdapijnSinSj-\lambdapijSiSj   \righto)^2  \righto)  } \leq \log \sqrt{\taskin\metan}.
\end{equation}
By combining~\eqref{eq:pf-of-tail-bound-one-step-1}-\eqref{eq:pf-of-tail-one-step-conc}, we get
\begin{equation}
\frac{\taskin\metan-1}{2}\lefto(\pacbavgmetapoploss\!-\!\pacbavgtrainloss\righto)^2 \!\leq\!  \relent{Q}{P} + \log \frac{\sqrt{\taskin\metan}}\delta.
\end{equation}
The desired result now follows after some arithmetic.

\end{proof}

\subsection{Proofs for Section~\ref{sec:expressiveness}}\label{sec:proof-expressive}

\begin{proof}[Proof of Corollary~\ref{cor:rep-learning-bound}]
We begin with~\eqref{eq:cor-rep-learning-bound-two-step}.
To establish this inequality, we bound the two sums on the right-hand side of~\eqref{eq:thm-two-step-sqrt} separately.
First, by Jensen's inequality, we find that
\begin{equation}\label{eq:pf-of-cor-4-task-term-1}
\avgij\Ex{\metasupersample,\subsetchoicei_j}{\sqrt{I^{\metasupersample,\subsetchoicei_j\!}\lefto(\!\lambdai_{j,\subsetchoicei_j};\metasubsetchoice_i\!\righto)}} \leq \sqrt{\avgij I\lefto(\!\lambdaarg{i}_{j,\subsetchoicei_j};\metasubsetchoice_i \vert \metasupersample,\subsetchoice\righto)}.
\end{equation}
Let~$\superexamplearg{i,k}_{j,l}$ denote the projection of~$\supersamplearg{i,k}_{j,l}$ onto~$\examplespace$, i.e.,~$\superexamplearg{i,k}_{j,l}$ contains the unlabelled instances from~$\metasupersample$.
The notation for~$\superexamplearg{i,k}_{j,l}$ is inherited from the notation for~$\supersamplearg{i,k}_{j,l}$ introduced in Section~\ref{sec:notation}.
Let~$f(\baselearner(   \supersample^{(i,k)}_{\subsetchoicei}, \randomness_i,\metalearner( \metatrainset,\metarandomness) ) , \cdot)$ denote the function from~$\mathcal F$ that is selected by~$\metalearner$ and~$\baselearner$ for task~$(i,k)$ on the basis of~$(\pacbtraindatan,\randomness,\metarandomness)$. 
We let~$\Farg{i,k}_{j,\subsetchoicei_j}$ denote the predicted label that the meta learner and the base learner produce for~$\superexamplearg{i,k}_{j,\subsetchoicei_j}$.
Furthermore, we let~$\Farg{i}_{j,\subsetchoicei_j}=(\Farg{i,0}_{j,\subsetchoicei_j}, \Farg{i,1}_{j,\subsetchoicei_j})$.
Again,~$\Farg{i,k}_{j,l}$ inherits the notational conventions that we use for~$\supersamplearg{i,k}_{j,l}$.
Note that, given~$\metasupersample$, the losses~$\lambdaarg{i}_{j,\subsetchoicei_j}$ are a function of~$\Farg{i}_{j,\subsetchoicei_j}$.
Thus, by the data-processing inequality,
\begin{equation}
\sqrt{\avgij I\lefto(\!\lambdaarg{i}_{j,\subsetchoicei_j};\metasubsetchoice_i \vert \metasupersample,\subsetchoice\righto)} \leq \sqrt{\avgij I\lefto(\!\Farg{i}_{j,\subsetchoicei_j};\metasubsetchoice_i \vert \metasupersample,\subsetchoice\righto)}.
\end{equation}
Next, Let~$h(\metalearner( \metatrainset,\metarandomness)  ,\cdot)$ denote the function from~$\mathcal H$ that is selected by~$\metalearner$ on the basis of~$(\pacbtraindatan,\randomness,\metarandomness)$.
We denote the representation that the meta learner induces on~$\superexamplearg{i}_{j,\subsetchoicei_j}$ as~$\Harg{i}_{j,\subsetchoicei_j}$, the elements of which is given by, for~$k\in\{0,1\}$,
\begin{equation}
\Harg{i,k}_{j,\subsetchoicei_j} = 
h(\metalearner( \metatrainset,\metarandomness)  ,\superexamplearg{i,k}_{j,\subsetchoicei_j}).
\end{equation}
Note that, given~$\metasupersample$,~$\subsetchoicei$ and~$\randomness_i$, the predictions in~$\Farg{i}_{j,\subsetchoicei_j}$ are a deterministic function of the intermediate representations~$\Harg{i}_{j,\subsetchoicei_j}$.
Therefore, using the independence of~$\randomness_i$ and~$\metasubsetchoice_i$,
\begin{align}
\sqrt{\avgij I\lefto(\!\Farg{i}_{j,\subsetchoicei_j};\metasubsetchoice_i \vert \metasupersample,\subsetchoice\righto)}  &\leq  \sqrt{\avgij I\lefto(\!\Farg{i}_{j,\subsetchoicei_j};\metasubsetchoice_i \vert \metasupersample,\subsetchoice,\randomness_i\righto)} \\
&\leq  \sqrt{\avgij I\lefto(\!\Harg{i}_{j,\subsetchoicei_j};\metasubsetchoice_i \vert \metasupersample,\subsetchoice\righto)} ,
\end{align}
where~$\randomness_i$ disappears from the conditioning due to the independence of~$\Harg{i}_{j,\subsetchoicei_j}$ and~$\randomness_i$.
Next, by adding random variables and using Lemma~\ref{lem:full-sample}, we get
\begin{align}
\sqrt{\avgij I\lefto(\!\Harg{i}_{j,\subsetchoicei_j};\metasubsetchoice_i \vert \metasupersample,\subsetchoice\righto)}  &\leq \sqrt{\avgij I\lefto(\!\Harg{}_{j,\subsetchoicei_j};\metasubsetchoice_i \vert \metasupersample,\subsetchoice\righto)}  \\
&\leq \sqrt{ \frac{\avgj I\lefto(\!\Harg{}_{j,\subsetchoice_j};\metasubsetchoice \vert \metasupersample,\subsetchoice\righto)}{\metan}}  .
\end{align}
For a given~$j$,~$\metasupersample$, and~$\supersubsetchoice$, the~$2\metan$ inputs that give rise to~$\Harg{}_{j,\subsetchoicei_j}$ are fixed.
Thus, the number of possible different values that~$\Harg{}_{j,\subsetchoicei_j}$ can take is at most~$\gH(2\metan)$, where~$\gH(\cdot)$ is the growth function of~$\mathcal H$.
From this, it follows that
\begin{align}\label{eq:indepofjinderiv}
 I\lefto( \Harg{}_{j,\subsetchoice_j};\metasubsetchoice\vert\metasupersample,\supersubsetchoice\righto) &\leq \entropyH\lefto(\Harg{}_{j,\subsetchoice_j}\vert \metasupersample,\supersubsetchoice\righto) \\
 &\leq \log \gH(2\metan) \\
 &\leq d_N\log \left(\binom{N}{2}\frac{2e\metan}{d_N} \right).
\end{align}
Here,~$\entropyH(\Harg{}_{j,\subsetchoice_j}\vert \metasupersample,\supersubsetchoice)$ denotes the conditional entropy of~$\Harg{}_{j,\subsetchoice_j}$ given~$(\metasupersample,\supersubsetchoice)$, and the last inequality follows from Lemma~\ref{lem:sauer-shelah}.
Since~\eqref{eq:indepofjinderiv} does not depend on~$j$, we find that
\begin{equation}\label{eq:pf-of-cor-4-task-term-last}
\sqrt{ \frac{ \avgj I\lefto(\Harg{}_{j,\subsetchoice_j};\metasubsetchoice \vert \metasupersample,\subsetchoice\righto)}{\metan}  } \leq \sqrt{ \frac{d_N\log \left(\binom{N}{2}\frac{2e\metan}{d_N} \right)}{\metan} }.
\end{equation}
By combining~\eqref{eq:pf-of-cor-4-task-term-1}-\eqref{eq:pf-of-cor-4-task-term-last}, we get
\begin{equation}\label{eq:pf-of-cor-4-task-term-final-bound}
\avgij\Ex{\metasupersample,\subsetchoicei_j}{\sqrt{I^{\metasupersample,\subsetchoicei_j\!}\lefto(\!\lambdai_{j,\subsetchoicei_j};\metasubsetchoice_i\!\righto)}} \leq \sqrt{ \frac{d_N\log \left(\binom{N}{2}\frac{2e\metan}{d_N} \right)}{\metan} }.
\end{equation}

Next, we turn to the second sum on the right-hand side of~\eqref{eq:thm-two-step-sqrt}.
First, by Jensen's inequality,
\begin{equation}\label{eq:pf-of-cor-4-environ-term-1}
\avgij\!\! \Ex{\metasupersample,\metasubsetchoice_i\!\!}{\sqrt{I^{\metasupersample,\metasubsetchoice_i\!}\lefto(\!\lambdaarg{i,\compl{\metasubsetchoice_i}}_j;\subsetchoiceinSi_j\!\righto)}} \leq  \sqrt{\avgij I\lefto(\!\lambdaarg{i,\compl{\metasubsetchoice_i}}_j;\subsetchoiceinSi_j\vert \metasupersample,\metasubsetchoice_i  \righto)}.
\end{equation}
Note that, given~$\metasupersample$, the losses~$\lambdaarg{i,\compl{\metasubsetchoice_i}}_j$ are a function of~$\Farg{i,\compl{\metasubsetchoice_i}}_j$.
Therefore,
\begin{align}
\sqrt{\avgij I\lefto(\!\lambdaarg{i,\compl{\metasubsetchoice_i}}_j;\subsetchoiceinSi_j\vert \metasupersample,\metasubsetchoice_i  \righto)} &\leq \sqrt{\avgij I\lefto(\!\Farg{i,\compl{\metasubsetchoice_i}}_j;\subsetchoiceinSi_j\vert \metasupersample,\metasubsetchoice_i  \righto)}  \\
&\leq \sqrt{\avgij I\lefto(\!\Farg{i,\compl{\metasubsetchoice_i}};\subsetchoiceinSi_j\vert \metasupersample,\metasubsetchoice_i  \righto)}  ,
\end{align}
where we used the fact that adding random variables cannot decrease mutual information.
By the independence of~$\subsetchoiceinSi_j$ for different~$j$ and Lemma~\ref{lem:full-sample},
\begin{align}
\sqrt{\avgij I\lefto(\!\Farg{i,\compl{\metasubsetchoice_i}};\subsetchoiceinSi_j\vert \metasupersample,\metasubsetchoice_i  \righto)}  \leq \sqrt{ \frac{\avgi I\lefto(\!\Farg{i,\compl{\metasubsetchoice_i}};\subsetchoiceinSi\vert \metasupersample,\metasubsetchoice_i  \righto)}{\taskin}} .
\end{align}
Now, note that given~$\metasupersample$ and~$\metasubsetchoice_i$, the~$2\taskin$ inputs that give rise to~$\Farg{i,\compl{\metasubsetchoice_i}}$ are fixed.
Recall that~$\gF(\cdot)$ denotes the growth function of~$\mathcal F$.
Then, 
\begin{align}
\sqrt{ \frac{\avgi I\lefto(\!\Farg{i,\compl{\metasubsetchoice_i}};\subsetchoiceinSi\vert \metasupersample,\metasubsetchoice_i  \righto)}{\taskin}} &\leq
\sqrt{ \frac{\avgi \entropyH\lefto(\!\Farg{i,\compl{\metasubsetchoice_i}}\vert \metasupersample,\metasubsetchoice_i  \righto)}{\taskin}}  \\
&\leq
\sqrt{ \frac{\avgi \gF(2\taskin)}{\taskin}} \\
&\leq \sqrt{ \frac{\avgi \dVC\log \lefto(\frac{2e\taskin}{\dVC} \righto)}{\taskin}} ,\label{eq:pf-of-cor-4-environ-term-last}
\end{align}
where we used Lemma~\ref{lem:sauer-shelah}.
By combining~\eqref{eq:pf-of-cor-4-environ-term-1}-\eqref{eq:pf-of-cor-4-environ-term-last}, we get
\begin{equation}\label{eq:pf-of-cor-4-environ-term-final-bound}
\avgij\!\! \Ex{\metasupersample,\metasubsetchoice_i\!\!}{\sqrt{I^{\metasupersample,\metasubsetchoice_i\!}\lefto(\!\lambdaarg{i,\compl{\metasubsetchoice_i}}_j;\subsetchoiceinSi_j\!\righto)}} \leq   \sqrt{ \frac{\avgi \dVC\log \lefto(\frac{2e\taskin}{\dVC} \righto)}{\taskin}} .
\end{equation}
The result in~\eqref{eq:cor-rep-learning-bound-two-step} now follows by combining~\eqref{eq:thm-two-step-sqrt},~\eqref{eq:pf-of-cor-4-task-term-final-bound} and~\eqref{eq:pf-of-cor-4-environ-term-final-bound}.
We now turn to~\eqref{eq:cor-rep-learning-bound-one-step}.
First, by Jensen's inequality,
\begin{align}\label{eq:pf-of-cor-4-one-step-1}
\avgij \Ex{\metasupersample}{\sqrt{ 2 I^{\metasupersample}(\lambdai_j; \metasubsetchoice_i,\subsetchoicei_j) }} &\leq  \sqrt{\avgij 2 I(\lambdai_j; \metasubsetchoice_i,\subsetchoicei_j\vert \metasupersample) }\\
&\leq  \sqrt{\avgij 2 I(\lambdahat; \metasubsetchoice_i,\subsetchoicei_j\vert \metasupersample) },
\end{align}
where in the second step, we used that adding random variables does not decrease mutual information.
Next, by the independence of the~$\subsetchoicei_j$ over~$j$ and of the~$(\metasubsetchoice_i,\subsetchoicei)$ over~$i$,
\begin{align}
\sqrt{\avgij 2 I(\lambdahat; \metasubsetchoice_i,\subsetchoicei_j\vert \metasupersample) }  &\leq \sqrt{ \frac{2 I(\lambdahat; \metasubsetchoice,\subsetchoice\vert \metasupersample)}{\taskin\metan} } .
\end{align}
Now, note that given~$\metasupersample$, the losses~$\lambdahat$ are a function of the predictions~$\Fhat=\{\Farg{i}_j \}_{j=1:\taskin}^{i=1:\metan}$.
Hence,
\begin{align}
\sqrt{ \frac{2 I(\lambdahat; \metasubsetchoice,\subsetchoice\vert \metasupersample)}{\taskin\metan} }   &\leq   \sqrt{ \frac{2 I(\Fhat; \metasubsetchoice,\subsetchoice\vert \metasupersample)}{\taskin\metan} } \\
&\leq  \sqrt{ \frac{2 I(\Fhat,\Hhat; \metasubsetchoice,\subsetchoice\vert \metasupersample)}{\taskin\metan} } 
\end{align}
where~$\Hhat=\{\Harg{i}_j\}_{j=1:\taskin}^{i=1:\metan}$ and the second step follows by adding random variables.
By the chain rule,
\begin{align}
\sqrt{ \frac{2 I(\Fhat,\Hhat; \metasubsetchoice,\subsetchoice\vert \metasupersample)}{\taskin\metan} } 
&\leq  \sqrt{ \frac{2 I(\Hhat; \metasubsetchoice,\subsetchoice\vert \metasupersample) + 2I(\Fhat; \metasubsetchoice,\subsetchoice\vert \metasupersample,\Hhat)}{\taskin\metan} } \\
&\leq \sqrt{ \frac{2 \entropyH\lefto(\Hhat\vert \metasupersample\righto) + 2\entropyH\lefto(\Fhat \vert \metasupersample,\Hhat\righto)}{\taskin\metan} } .
\end{align}
Since~$\Hhat$ is given by the elementwise application of some~$h\in \mathcal H$ to~$\metasuperexample=\{\superexamplearg{i}_j \}_{j=1:\taskin}^{1:\metan}$, it can take at most~$\gH(4\taskin\metan)$ different values, similar to previous arguments.
This implies that
\begin{align}
\entropyH\lefto( \Hhat\vert \metasupersample  \righto) \leq \log(\gH(4\taskin\metan ))\leq d_N\log \lefto(\binom{N}{2}\frac{4e\taskin\metan}{d_N} \righto),
\end{align}
where the last inequality is again due to Lemma~\ref{lem:sauer-shelah}.
Given~$\Hhat$ and~$\metasupersample$, the predictions~$\Fhat$ can take at most~$(\gF(2\taskin))^{2\metan}$ different values, since the~$2\taskin$ inputs to each of the~$2\metan$ task-specific functions are fixed.
This implies that
\begin{align}\label{eq:pf-of-cor-4-one-step-last}
\entropyH\lefto( \Fhat \vert \Hhat, \metasupersample \righto) \leq 2\metan \log(\gF(2\taskin ))\leq 2\metan\dVC\log \lefto(\frac{2e\taskin}{\dVC} \righto),
\end{align}
where we again used Lemma~\ref{lem:sauer-shelah}.
The desired result follows by combining~\eqref{eq:thm-one-step-sqrt} with~\eqref{eq:pf-of-cor-4-one-step-1}-\eqref{eq:pf-of-cor-4-one-step-last}.

\end{proof}

\begin{proof}[Proof of Corollary~\ref{cor:rep-learning-bound-interp}]

By the same steps as in~\eqref{eq:pf-of-cor-4-one-step-1}-\eqref{eq:pf-of-cor-4-one-step-last}, we find that
\begin{equation}
\avgij I(\lambdai_j;\metasubsetchoice_i,\subsetchoicei_j\vert \metasupersample) \leq d_N\log \lefto(\binom{N}{2}\frac{4e\taskin\metan}{d_N} \righto)  +  2\metan\dVC\log \lefto(\frac{2e\taskin}{\dVC} \righto).
\end{equation}
By combining this with~\eqref{eq:cor-interp-one-step-kl}, we find that
\begin{align}
\avgmetapoploss \leq  \frac{4 d_N\log \lefto(\binom{N}{2}\frac{4e\taskin\metan}{d_N} \righto)  +  8\metan\dVC\log \lefto(\frac{2e\taskin}{\dVC} \righto)}{\taskin\metan}.
\end{align}
This establishes the desired result.
\end{proof}

\begin{proof}[Proof of Corollary~\ref{cor:rep-learning-high-probability}]

First, we establish~\eqref{eq:high-probability-cor-rep-learning-bound-two-step}.
As mentioned in the proof of Theorem~\ref{thm:high-prob-two-step-sqrt}, we start the derivation from the more general bound given in~\eqref{eq:pf-of-tail-combined-level-bound} rather than the simplified bound given in~\eqref{eq:thm-high-prob-two-step-sqrt}.
\looseness=-1 We begin by bounding~$\relent{Q_{j,\subsetchoice_j}}{P_{j,\subsetchoice_j}}$.
Recall that~$Q_{j,\subsetchoice_j}$ denotes the conditional distribution of~$\lambda_{j,\subsetchoice_j}$ given~$(\metasupersample, \metasubsetchoice,\subsetchoice )$ and~$P_{j,\subsetchoice_j}=\Ex{\metasubsetchoice}{Q_{j,\subsetchoice_j}}$.
Let~$\lambda_{j,\subsetchoice_j}$ be distributed according to~$Q_{j,\subsetchoice_j}$.
By Jensen's inequality,
\begin{align}\label{eq:pf-of-rep-learning-two-step-environ-level-1}
\relent{Q_{j,\subsetchoice_j}}{P_{j,\subsetchoice_j}} = \Ex{\lambda_{j,\subsetchoice_j}}{\log\frac{Q_{j,\subsetchoice_j}(\lambda_{j,\subsetchoice_j})}{P_{j,\subsetchoice_j}(\lambda_{j,\subsetchoice_j})}} \leq  \log\Ex{\lambda_{j,\subsetchoice_j}}{  \frac{Q_{j,\subsetchoice_j}(\lambda_{j,\subsetchoice_j})}{P_{j,\subsetchoice_j}(\lambda_{j,\subsetchoice_j})}  }. 
\end{align}
By Markov's inequality, with probability at least~$1-\delta$ under the draw of~$\pacbtraindata$,
\begin{align}
\relent{Q_{j,\subsetchoice_j}}{P_{j,\subsetchoice_j}} = \Ex{\lambda_{j,\subsetchoice_j}}{\log\frac{Q_{j,\subsetchoice_j}(\lambda_{j,\subsetchoice_j})}{P_{j,\subsetchoice_j}(\lambda_{j,\subsetchoice_j})}} \leq  \log\lefto( \frac1\delta\Ex{\lambda_{j,\subsetchoice_j},\pacbtraindatan}{  \frac{Q_{j,\subsetchoice_j}(\lambda_{j,\subsetchoice_j})}{P_{j,\subsetchoice_j}(\lambda_{j,\subsetchoice_j})}  } \righto). 
\end{align}
Since~$\lambda_{j,\subsetchoice_j}$ is a discrete random variable,~$Q_{j,\subsetchoice_j}(\lambda_{j,\subsetchoice_j})\leq 1$.
Hence,
\begin{align}
\log\lefto( \frac1\delta\Ex{\lambda_{j,\subsetchoice_j},\pacbtraindatan}{  \frac{Q_{j,\subsetchoice_j}(\lambda_{j,\subsetchoice_j})}{P_{j,\subsetchoice_j}(\lambda_{j,\subsetchoice_j})}  } \righto) \leq \log\lefto( \frac1\delta\Ex{\lambda_{j,\subsetchoice_j},\pacbtraindatan}{  \frac{1}{P_{j,\subsetchoice_j}(\lambda_{j,\subsetchoice_j})}  } \righto). 
\end{align}
Recall that~$\Ex{\metasubsetchoice}{Q_{j,\subsetchoice_j}}=P_{j,\subsetchoice_j}$.
Let~$\lambda'_{j,\subsetchoice_j}$ be distributed according to~$P_{j,\subsetchoice_j}$.
Since the argument of the expectation is now independent of~$\metasubsetchoice$,
\begin{align}
\log\lefto( \frac1\delta\Ex{\lambda_{j,\subsetchoice_j},\pacbtraindatan}{  \frac{1}{P_{j,\subsetchoice_j}(\lambda_{j,\subsetchoice_j})}  } \righto)  &=  \log\lefto( \frac1\delta\Ex{\lambda'_{j,\subsetchoice_j},\metasupersample,\subsetchoice}{  \frac{1}{P_{j,\subsetchoice_j}(\lambda_{j,\subsetchoice_j})}  } \righto)  \\
&\leq \log\lefto( \frac1\delta \sup_{\metasupersample,\subsetchoice} \Ex{\lambda'_{j,\subsetchoice_j}}{  \frac{1}{P_{j,\subsetchoice_j}(\lambda_{j,\subsetchoice_j})}  } \righto).
\end{align}
Now, let~$\Lambda_{j,\subsetchoice_j}(\metasupersample,\subsetchoice)$ denote the set of all possible values that~$\lambda'_{j,\subsetchoice_j}$ can take given~$(\metasupersample,\subsetchoice)$.
Then,
\begin{align}
\log\lefto( \frac1\delta \sup_{\metasupersample,\subsetchoice} \Ex{\lambda'_{j,\subsetchoice_j}}{  \frac{1}{P_{j,\subsetchoice_j}(\lambda_{j,\subsetchoice_j})}  } \righto) &= \log\lefto( \frac1\delta \sup_{\metasupersample,\subsetchoice} \sum_{\lambda'_{j,\subsetchoice_j}\in\Lambda_{j,\subsetchoice_j}(\metasupersample,\subsetchoice)}  \frac{P_{j,\subsetchoice_j}(\lambda_{j,\subsetchoice_j})}{P_{j,\subsetchoice_j}(\lambda_{j,\subsetchoice_j})}   \righto)\\
&= \log\lefto( \frac1\delta \sup_{\metasupersample,\subsetchoice} \abs{\Lambda_{j,\subsetchoice_j}(\metasupersample,\subsetchoice)} \righto).\label{eq:pf-of-rep-learning-two-step-environ-level-last}
\end{align}
Now, note that since~$\lambda_{j,\subsetchoice_j}$ is averaged over~$\randomness$, it is a function of~$\Harg{}_{j,\subsetchoice_j}$ given~$(\metasupersample,\subsetchoice)$.
Furthermore, the inputs~$\superexample_{j,\subsetchoice_j}$ are fixed.
Therefore, as argued in the proof of Corollary~\ref{cor:rep-learning-bound}, the number of different values that~$\Harg{}_{j,\subsetchoice_j}$ can take given~$(\metasupersample,\subsetchoice)$ is at most~$\gH(2\metan)$.\footnote{If we had instead used the result of Remark~\ref{rem:alternative-high-prob-two-step-sqrt} and followed analogous steps, we would instead get~$\gH(2\taskin\metan)$ as an upper bound of the KL divergence.}
Thus, by combining~\eqref{eq:pf-of-rep-learning-two-step-environ-level-1}-\eqref{eq:pf-of-rep-learning-two-step-environ-level-last}, we get
\begin{align}\label{eq:pf-of-rep-learning-two-step-environ-level-kl-term-bound}
\relent{Q_{j,\subsetchoice_j}}{P_{j,\subsetchoice_j}} \leq \log\lefto(\frac{\gH(2\metan)}{\delta} \righto) \leq d_N\log \left(\binom{N}{2}\frac{2e\metan}{d_N} \right) + \log\frac1\delta.
\end{align}

Next, we turn to~$\relent{Q^{\compl{\metasubsetchoice}}}{P^{\compl{\metasubsetchoice}}}$.
Recall that~$Q^{\compl{\metasubsetchoice}}$ denotes the conditional distribution of~$\lambdaarg{\compl{\metasubsetchoice}}$ given~$(\metasupersample, \metasubsetchoice,\subsetchoice )$, and~$P^{\compl{\metasubsetchoice}}=\Ex{\subsetchoice^{\compl{\metasubsetchoice}}}{Q^{\compl{\metasubsetchoice}}}$.
Let~$\lambdaarg{\compl{\metasubsetchoice}}$ be distributed according to~$Q^{\compl{\metasubsetchoice}}$.
Again, by Jensen's inequality,
\begin{align}\label{eq:pf-of-rep-learning-two-step-task-level-1}
\relent{Q^{\compl{\metasubsetchoice}}}{P^{\compl{\metasubsetchoice}}} = \Ex{\lambdaarg{\compl{\metasubsetchoice}}}{\log\frac{Q^{\compl{\metasubsetchoice}}(\lambdaarg{\compl{\metasubsetchoice}})}{P^{\compl{\metasubsetchoice}}(\lambdaarg{\compl{\metasubsetchoice}})}} \leq  \log\Ex{\lambdaarg{\compl{\metasubsetchoice}}}{  \frac{Q^{\compl{\metasubsetchoice}}(\lambdaarg{\compl{\metasubsetchoice}})}{P^{\compl{\metasubsetchoice}}(\lambdaarg{\compl{\metasubsetchoice}})}  }. 
\end{align}
By Markov's inequality, with probability at least~$1-\delta$ under the draw of~$\pacbtraindata$,
\begin{align}
\relent{Q^{\compl{\metasubsetchoice}}}{P^{\compl{\metasubsetchoice}}} = \Ex{\lambdaarg{\compl{\metasubsetchoice}}}{\log\frac{Q^{\compl{\metasubsetchoice}}(\lambdaarg{\compl{\metasubsetchoice}})}{P^{\compl{\metasubsetchoice}}(\lambdaarg{\compl{\metasubsetchoice}})}} \leq  \log\lefto( \frac1\delta\Ex{\lambdaarg{\compl{\metasubsetchoice}},\pacbtraindatan}{  \frac{Q^{\compl{\metasubsetchoice}}(\lambdaarg{\compl{\metasubsetchoice}})}{P^{\compl{\metasubsetchoice}}(\lambdaarg{\compl{\metasubsetchoice}})}  } \righto). 
\end{align}
Since~$\lambdaarg{\compl{\metasubsetchoice}}$ is a discrete random variable,~$Q^{\compl{\metasubsetchoice}}(\lambdaarg{\compl{\metasubsetchoice}})\leq 1$.
Therefore,
\begin{align}
\log\lefto( \frac1\delta\Ex{\lambdaarg{\compl{\metasubsetchoice}},\pacbtraindatan}{  \frac{Q^{\compl{\metasubsetchoice}}(\lambdaarg{\compl{\metasubsetchoice}})}{P^{\compl{\metasubsetchoice}}(\lambdaarg{\compl{\metasubsetchoice}})}  } \righto) \leq \log\lefto( \frac1\delta\Ex{\lambdaarg{\compl{\metasubsetchoice}},\pacbtraindatan}{  \frac{1}{P^{\compl{\metasubsetchoice}}(\lambdaarg{\compl{\metasubsetchoice}})}  } \righto). 
\end{align}
Now, let~${\lambda'}^{\compl{\metasubsetchoice}}$ be distributed according to~$P^{\compl{\metasubsetchoice}}$.
Since the argument of the expectation is now independent of~$\subsetchoice^{\compl{\metasubsetchoice}}$,
\begin{align}
\log\lefto( \frac1\delta\Ex{\lambdaarg{\compl{\metasubsetchoice}},\pacbtraindatan}{  \frac{1}{P^{\compl{\metasubsetchoice}}(\lambdaarg{\compl{\metasubsetchoice}})}  } \righto)  &=  \log\lefto( \frac1\delta\Ex{{\lambda'}^{\compl{\metasubsetchoice}},\metasupersample,\metasubsetchoice,\subsetchoice^{\metasubsetchoice}}{  \frac{1}{P^{\compl{\metasubsetchoice}}(\lambdaarg{\compl{\metasubsetchoice}})}  } \righto)  \\
&\leq \log\lefto( \frac1\delta \sup_{\metasupersample,\metasubsetchoice,\subsetchoice^{\metasubsetchoice}} \Ex{{\lambda'}^{\compl{\metasubsetchoice}}}{  \frac{1}{P^{\compl{\metasubsetchoice}}(\lambdaarg{\compl{\metasubsetchoice}})}  } \righto).
\end{align}
Now, let~$\Lambda^{\compl{\metasubsetchoice}}(\metasupersample,\subsetchoice)$ denote the set of all possible values that~${\lambda'}^{\compl{\metasubsetchoice}}$ can take given~$(\metasupersample,\metasubsetchoice,\subsetchoice^{\metasubsetchoice})$.
Then,
\begin{align}
\log\lefto( \frac1\delta \sup_{\metasupersample,\metasubsetchoice,\subsetchoice^{\metasubsetchoice}} \Ex{{\lambda'}^{\compl{\metasubsetchoice}}}{  \frac{1}{P^{\compl{\metasubsetchoice}}(\lambdaarg{\compl{\metasubsetchoice}})}  } \righto) 
&= \log\lefto( \frac1\delta \sup_{\metasupersample,\metasubsetchoice,\subsetchoice^{\metasubsetchoice}} \sum_{{\lambda'}^{\compl{\metasubsetchoice}}\in\Lambda^{\compl{\metasubsetchoice}}(\metasupersample,\subsetchoice)}  \frac{P^{\compl{\metasubsetchoice}}(\lambdaarg{\compl{\metasubsetchoice}})}{P^{\compl{\metasubsetchoice}}(\lambdaarg{\compl{\metasubsetchoice}})}   \righto)\\
&= \log\lefto( \frac1\delta \sup_{\metasupersample,\metasubsetchoice,\subsetchoice^{\metasubsetchoice}} \abs{\Lambda^{\compl{\metasubsetchoice}}(\metasupersample,\subsetchoice)} \righto).\label{eq:pf-of-rep-learning-two-step-task-level-last}
\end{align}
Note that, given~$\metasupersample$, the losses~$\lambdaarg{\compl{i,\metasubsetchoice_i}}$ are a function of the predictions~$\Farg{i,\compl{\metasubsetchoice_i}}$.
Furthermore, given~$(\metasupersample,\metasubsetchoice,\subsetchoice^{\metasubsetchoice})$, the inputs~$\Harg{i,\compl{\metasubsetchoice_i}}$ are fixed.
This is the case since~$\hyperparam$ is independent from~$\subsetchoice^{\compl{\metasubsetchoice}}$.
Thus, similar to previous arguments, given~$(\metasupersample,\metasubsetchoice,\subsetchoice^{\metasubsetchoice})$,~$\Farg{i,\compl{\metasubsetchoice_i}}$ can take at most~$\gF(2\taskin)$ different values for each~$i$.
Therefore,~$\Farg{\compl{\metasubsetchoice}}$ can take at most~$\gF(2\taskin)^{\metan}$ values.
Thus, by combining~\eqref{eq:pf-of-rep-learning-two-step-task-level-1}-\eqref{eq:pf-of-rep-learning-two-step-task-level-last}, we get
\begin{align}\label{eq:pf-of-rep-learning-two-step-task-level-kl-term-bound}
\relent{Q^{\compl{\metasubsetchoice}}}{P^{\compl{\metasubsetchoice}}} \leq \log\lefto(\frac{\gF(2\taskin)^{\metan}}{\delta} \righto) \leq \metan \dVC\log \left(\frac{2e\taskin}{\dVC} \right) + \log\frac1\delta,
\end{align}
where we used Lemma~\ref{lem:sauer-shelah}.
Thus, by using a union bound, we can combine~\eqref{eq:pf-of-tail-combined-level-bound},~\eqref{eq:pf-of-rep-learning-two-step-environ-level-kl-term-bound}, and~\eqref{eq:pf-of-rep-learning-two-step-task-level-kl-term-bound}, with~$\delta\rightarrow\delta/3$, to conclude that with probability at least~$1-\delta$ under the draw of~$\pacbtraindata$,
\begin{multline}
\abs{\avgmetapoploss\!-\!\avgtrainloss} \!\leq\! \sqrt{ \frac{ 2  d_N\log \lefto(\binom{N}{2}\frac{2e\metan}{d_N} \righto) + \log\frac3\delta  \!+\!  2\log\lefto(\frac{6\taskin\sqrt {\metan}}{\delta}\righto)}{\metan-1} } \\
+ \sqrt{ \frac{  2\metan \dVC\log \lefto(\frac{2e\taskin}{\dVC} \righto) + \log\frac3\delta \!+\! 2\log\lefto( \frac{6\sqrt {\taskin\metan}}{\delta} \righto) }{\taskin\metan-1} }.\!\!\!
\end{multline}
Under the assumption that~$\taskin,\metan\geq 2$, by similar arguments as in the proof of Theorem~\ref{thm:high-prob-two-step-sqrt}, we find that, for some constants~$C_1$ and~$C_2$,
\begin{equation}\label{eq:app-cor-rep-learning-bound-two-step-correct}
\abs{\avgmetapoploss\!-\!\avgtrainloss} \!\leq\! C_1\sqrt{ \frac{   d_N\log \lefto(\binom{N}{2}\frac{\metan}{d_N} \righto) +   \log\lefto(\frac{\taskin\sqrt {\metan}}{\delta}\righto)}{\metan} }
+ C_2\sqrt{ \frac{   \dVC\log \lefto(\frac{\taskin}{\dVC} \righto) +  \log\lefto( \frac{\sqrt {\taskin}}{\delta} \righto) }{\taskin} }.\!\!\!
\end{equation}
This establishes~\eqref{eq:high-probability-cor-rep-learning-bound-two-step}.

We now turn to~\eqref{eq:high-probability-cor-rep-learning-bound-one-step}.
Let~$\lambdahat$ be distributed according to~$Q$.
First, by Jensen's inequality,
\begin{equation}
\relent{Q}{ P } = \Ex{\lambdahat}{\log \frac{ Q(\lambdahat) }{ P(\lambdahat) } }  \leq  \log \Ex{\lambdahat}{ \frac{ Q(\lambdahat) }{P(\lambdahat) } } .
\end{equation}
By Markov's inequality, with probability at least~$1-\delta$ under the draw of~$\pacbtraindata$,
\begin{equation}
\relent{Q}{ P } = \Ex{\lambdahat}{\log \frac{ Q(\lambdahat) }{P(\lambdahat) } }  \leq  \log\lefto(\frac1\delta \Ex{\lambdahat,\pacbtraindatan}{ \frac{ Q(\lambdahat) }{ P(\lambdahat) } } \righto).
\end{equation}
Since~$\lambdahat$ is a discrete random variable,~$Q(\lambdahat)\leq 1$.
Hence,
\begin{equation}
\relent{Q}{ P } = \Ex{\lambdahat}{\log \frac{ Q(\lambdahat) }{ P(\lambdahat) } }  \leq  \log \lefto(\frac1\delta \Ex{\lambdahat,\pacbtraindatan}{ \frac{ 1 }{ P(\lambdahat) } } \righto) .
\end{equation}
Recall that~$\Ex{\metasubsetchoice,\subsetchoice}{Q}=P$.
Let~$\lambdahat'$ be distributed according to~$P$.
Since the argument of the expectation is now independent of~$(\metasubsetchoice,\subsetchoice)$,
\begin{align}
\log \lefto(\frac1\delta \Ex{\lambdahat,\pacbtraindatan}{ \frac{ 1 }{ P(\lambdahat) } } \righto)  &\leq \log \lefto(\frac1\delta \Ex{\lambdahat',\metasupersample}{ \frac{ 1 }{ P(\lambdahat') } } \righto) \\
&\leq \log \lefto(\frac1\delta \sup_{\metasupersample}\Ex{\lambdahat'}{ \frac{ 1 }{ P(\lambdahat') } } \righto)  .
\end{align}
Let~$\Lambda(\metasupersample)$ denote the set of all possible values that~$\lambdahat'$ can take given~$\metasupersample$.
Then,
\begin{align}
\log \lefto(\frac1\delta \sup_{\metasupersample}\Ex{\lambdahat'}{ \frac{ 1 }{ P(\lambdahat') } } \righto) &= \log \lefto(\frac1\delta \sup_{\metasupersample}\sum_{\lambdahat'\in\Lambda(\metasupersample)}\frac{ P(\lambdahat') }{ P(\lambdahat') } \righto) \\
&= \log \lefto(\frac1\delta \sup_{\metasupersample}\abs{\Lambda(\metasupersample)} \righto).
\end{align}
Since the map from predictions to losses is surjective,~$\abs{\Lambda(\metasupersample)}$ is bounded by the number of possible predictions~$F$ given~$\metasupersample$.
We can bound this as follows.
First, the number of possible different values for~$H$ given~$\metasupersample$ is at most~$\gH(4\taskin\metan)$.
Given a fixed~$H$, the number of possible values that~$F$ can take is at most~$(\gF(2\taskin))^{2\metan}$, since the~$2\taskin$ inputs to each of the~$2\metan$ task-specific functions are fixed.
Therefore, the total number of possible values for~$F$ given~$\metasupersample$ is at most~$\gH(4\taskin\metan)(\gF(2\taskin))^{2\metan}$.
Hence,
\begin{align}
\log \lefto(\frac1\delta \sup_{\metasupersample}\abs{\Lambda(\metasupersample)} \righto) &\leq  \log(\gH(4\taskin\metan)) + 2\metan\log(\gF(2\taskin)) +  \log\frac1\delta\\
&\leq d_N\log \lefto(\binom{N}{2}\frac{4e\taskin\metan}{d_N} \righto) + 2\metan\dVC\log \lefto(\frac{2e\taskin}{\dVC} \righto) + \log \frac1\delta
\end{align}
where we used Lemma~\ref{lem:sauer-shelah}.
Substituting this into~\eqref{eq:thm-high-prob-one-step-sqrt}, using a union bound and letting~$\delta\rightarrow\delta/2$, we find that with probability at least~$1-\delta$ under the draw of~$\pacbtraindata$,
\begin{multline}
\abs{\avgmetapoploss(\metasupersample,\metasubsetchoice,\subsetchoice)\!-\!\avgtrainloss(\metasupersample,\metasubsetchoice,\subsetchoice)}\!\\  \leq\! \!  {\sqrt{ \frac{2d_N\log \lefto(\binom{N}{2}\frac{4e\taskin\metan}{d_N} \righto) + 4\metan\dVC\log \lefto(\frac{2e\taskin}{\dVC} \righto) + 2\log \frac2\delta + 2\log\lefto( \frac{2\sqrt{\taskin\metan}}{\delta} \righto)}{\taskin\metan-1}}} .
\end{multline}
Assuming that~$\taskin,\metan\geq 2$, the desired result in~\eqref{eq:cor-rep-learning-bound-one-step} follows by upper-bounding constants by using similar arguments as in the proof of Theorem~\ref{thm:high-prob-two-step-sqrt}.

\end{proof}

\section{Bound for the Excess Risk}\label{app:excess}

We now turn to excess risk bounds. 
In Corollary~\ref{cor:excess-bound-spec-task}, we present the formal statement of the excess risk bound in Section~\ref{sec:excess-informal}.
In Corollary~\ref{cor:excess-random-task}, we state an excess risk bound for a randomly drawn new task, which obviates the need of a task diversity assumption.

In order to derive excess risk bounds, we need to introduce some technical tools.
First, we need to consider \textit{oracle} algorithms, that is, algorithms that output minimizers of the population loss.
Specifically, the oracle meta learner knows the task distribution~$\taskdistro$, while the oracle base learner knows the indexed set of in-task distributions~$\{D_\tau: \tau\in\taskspace\}$.
While these algorithms have access to the data distributions, and are thus not of practical interest, they are useful as a proof technique, and can be analyzed in the same way as realistic algorithms.
Second, in order to allow the oracle base learner to minimize the population loss for a given task, we need to extend the input to the base learner to include the identity of the task~$\tau\in\taskspace$.
Thus, the base learner is a mapping~$\baselearner: \samplespace^{\taskin}\times\taskspace\times \randomspace\times \hyperspace\rightarrow \paramspace$.
For the case of a base learner that minimizes the empirical risk, the task identity is irrelevant, so the input from~$\taskspace$ does not affect the output.
Conversely, for an oracle base learner, the training samples are irrelevant, so only the input from~$\taskspace$ affects the output.
Finally, our information-theoretic bounds pertain to a test loss, rather than the population loss.
While these are equal for average bounds, there is a small discrepancy for the high-probability bounds.
In order to handle excess risk bounds and oracle algorithms that depend on the population loss, we need to convert between the two by using a Hoeffding bound, as discussed in~\cite[Thm.~3]{hellstrom-20b}.
The extra terms that this additional step leads to are typically negligible compared to the dominant complexity terms.

For concreteness, we focus only on high-probability excess risk bounds derived on the basis of the one-step square-root bound in Corollary~\ref{cor:rep-learning-high-probability}.
However, note that excess risk bounds based on the other high-probability bounds can be obtained by suitably substituting these alternative bounds in the proofs.
Average excess risk bounds can also be derived by an analogous procedure.
First, using the task diversity assumption of~\cite{tripuraneni-20a}, we derive an excess risk bound for a fixed target task.

\begin{cor}\label{cor:excess-bound-spec-task}
Consider the setting of Corollary 6 and a fixed task~$\tau_0$.
Let~$Z^0\in \mathcal Z^{2\times m}$ be a matrix of~$2m$ samples generated independently according to~$\datadistroarg{\tau_0}$, the data distribution for task~$\tau_0$. 
Let~$S^0$ be an~$m$-dimensional random vector with elements generated independently from a~$\mathrm{Bern}(1/2)$ distribution, and let the training set~$Z^0_{S^0}$ and test set~$Z^0_{\compl{S^0}}$ be constructed in the same way as the training and test sets for tasks~$1,\dots,\metan$.
To simplify notation, let~$Z^{0,0}=Z^{0,1}=Z^0$.
Denote the population loss for the~$i$th observed task when using the base learner~$\baselearner'$ with the representation~$h'$ as
\begin{equation}
\excesspoparg{i,\baselearner',h'} = \Ex{R_i, \tilde Z\distas \datadistroarg{\tau_{i,\metasubsetchoice_i}}}{ \ell\lefto( \baselearner'\lefto( \supersamplearg{i,{\metasubsetchoice_i}}_{\subsetchoicei},\tau_{i,\metasubsetchoice_i},\randomness_i, h' \righto), \tilde Z \righto) }.
\end{equation}
Similarly, denote the population loss for the~$i$th unobserved task when using the base learner~$\baselearner'$ with the representation~$h'$ as
\begin{equation}
\excesspoparg{-i,\baselearner',h'} = \Ex{R_i, \tilde Z\distas \datadistroarg{\tau_{i,\compl{\metasubsetchoice_i}}}}{ \ell\lefto( \baselearner'\lefto( \supersamplearg{i,\compl{\metasubsetchoice_i}}_{\subsetchoicei},\tau_{i,\compl{\metasubsetchoice_i}},\randomness_i, h' \righto), \tilde Z \righto) }.
\end{equation}
Let~$\baselearner^*$ denote an oracle learner that satisfies~$\excesspoparg{i,\baselearner^*,h'} = \min_{\baselearner'} \excesspoparg{i,\baselearner',h'}$ for all~$h'$.
Assume that~$h^*=\argmin_{h'} \min_{\baselearner'} \excesspoparg{i,\baselearner',h'}$ for all~$\tau_{i,\metasubsetchoice_i}$ and that~$h^*\in \mathcal H$.
Thus, the same representation~$h^*$ is optimal for all tasks.
Let~$\metalearner$ and~$\baselearner$ be empirical risk minimizers, and let~$\hat h=\metalearner(\metarandomness,\metatrainset)$. %
Finally, assume that the supersample satisfies a task-diversity assumption, so that for some~$\nu$ and~$\epsilon$,
\begin{equation}
\sup_{\tau_0}\excesspoparg{0,\baselearner^*,\hat h} - \excesspoparg{0,\baselearner^*,h^*} \leq \nu^{-1} \left( \excesspoparg{-\oneton,\baselearner^*,\hat h} -\excesspoparg{-\oneton,\baselearner^*,h^*} \right) + \epsilon.
\end{equation}
Then, there exist constants~$C_1$ and~$C_2$ such that, with probability at least~$1-\delta$ under the draw of~$(\pacbtraindatan,Z^0,S^0)$, we have
\begin{multline}
\excesspoparg{0,\baselearner,\hat h} \!-\! \excesspoparg{0,\baselearner^*,h^*} 
\leq C_1\sqrt{ \frac{ \dVC \log\lefto(\frac{\sqrt m}{\dVC}\righto) \!+\! \log\lefto(\frac{\sqrt m}{\delta}\righto) }{m} } \\
+ C_2 \nu^{-1}\sqrt{ \frac{d_N \log \lefto(\binom{N}{2}\frac{\taskin\metan}{d_N} \righto) + \metan\dVC \log \lefto(\frac{\taskin}{\dVC} \righto) + \log\lefto( \frac{\sqrt{\taskin\metan}}{\delta}\righto)}{\taskin\metan} } + \epsilon.
\end{multline}
\end{cor}

\begin{proof}

We will use the following shorthands.
When using the algorithm~$\baselearner'$ for task~$i$ based on the representation~$h'$, we let~$\excesspoparg{i,\baselearner',h'}$ denote the population loss,~$\excesstrainarg{i,\baselearner',h'}$ denote the training loss, and~$\excesstestarg{i,\baselearner',h'}$ denote the test loss on a test set of the same size as the training set.
Formally,
\begin{align}
\excesspoparg{i,\baselearner',h'} &= \Ex{R_i, \tilde Z\distas \datadistroarg{\tau_{i,\metasubsetchoice_i}}}{ \ell\lefto( \baselearner'\lefto( \supersamplearg{i,{\metasubsetchoice_i}}_{\subsetchoicei},\tau_{i,\metasubsetchoice_i},\randomness_i, h' \righto), \tilde Z \righto) }, \\
\excesstrainarg{i,\baselearner',h'} &=  \avgj \Ex{R_i}{ \ell\lefto( \baselearner'\lefto( \supersamplearg{i,{\metasubsetchoice_i}}_{\subsetchoicei},\tau_{i,\metasubsetchoice_i},\randomness_i, h' \righto), \supersamplearg{i,{\metasubsetchoice_i}}_{j,\subsetchoicei_j} \righto) },\\
\excesstestarg{i,\baselearner',h'} &= \avgj \Ex{R_i}{ \ell\lefto( \baselearner'\lefto( \supersamplearg{i,{\metasubsetchoice_i}}_{\subsetchoicei},\tau_{i,\metasubsetchoice_i},\randomness_i, h' \righto), \supersamplearg{i,{\metasubsetchoice_i}}_{j,\compl{\subsetchoicei_j}} \righto) }.
\end{align}
As a shorthand, we let~$\excesspoparg{\oneton,\baselearner',h'}=\avgi \excesspoparg{i,\baselearner',h'} $, and we use the same convention for~$\excesstrainarg{\oneton,\baselearner',h'}$ and~$\excesstestarg{\oneton,\baselearner',h'}$.
Furthermore, to indicate losses on unobserved tasks we negate the task index.
Thus,
\begin{align}
\excesstrainarg{-i,\baselearner',h'} &=  \avgj \Ex{R_i}{ \ell\lefto( \baselearner'\lefto( \supersamplearg{i,\compl{\metasubsetchoice_i}}_{\subsetchoicei},\tau_{i,\compl{\metasubsetchoice_i}},\randomness_i, h' \righto), \supersamplearg{i,\compl{\metasubsetchoice_i}}_{j,\subsetchoicei_j} \righto) },
\end{align}
with analogous notation for the test and population losses.

The base learner~$\baselearner$ that we consider is an empirical risk minimizer, which satisfies for all~$h'$
\begin{align}
    \excesstrainarg{i,\baselearner,h'} = \min_{\baselearner'} \excesstrainarg{i,\baselearner',h'}.
\end{align}
For our analysis, we use an oracle learner~$\baselearner^*$, which outputs the minimizer of the population loss for the given task.
While this is not a realistic learning algorithm in practice, as it depends on the data distribution, it is useful as an analysis tool.
Formally, for all~$h'$,
\begin{equation}
\excesspoparg{i,\baselearner^*,h'} = \min_{\baselearner'} \excesspoparg{i,\baselearner',h'}.
\end{equation}
Finally, we let~$\hat h$ be a representation that minimizes the empirical risk over the~$\metan$ training tasks and~$h^*$ be an optimal representation, i.e.
\begin{align}
\hat h &\in \argmin_{h'} \excesstrainarg{\oneton, \baselearner, h'}. \\
h^* &\in \argmin_{h'} \excesspoparg{i,\baselearner^*,h'}.
\end{align}
By assumption,~$h^*$ is the same for any task~$\tau_{i,k}$.

In the proof, we need to convert between test losses and population losses.
By definition, test data is independent from the hypothesis, so standard concentration inequalities can be applied to bound the difference between the test and population loss.
The following lemma follows immediately from Hoeffding's inequality~\cite[Prop.~2.5]{wainwright19-a}, as argued in~\cite[Thm.~3]{hellstrom-20b}.
\begin{lem}\label{lem:hoeffding}
Let~$\excesstestarg{i,\baselearner',h'}$ be a test loss based on~$m$ samples.
Then, with probability at least~$1-\delta$,
\begin{equation}
   \abs{\excesstestarg{i,\baselearner',h'}-\excesspoparg{i,\baselearner',h'}} \leq \sqrt{ \frac{\log\frac{2}{\delta}}{2m} }.
\end{equation}
\end{lem}
\begin{proof}
The test loss~$\excesstestarg{i,\baselearner',h'}$ is the average of~$m$ independent samples of a bounded random variable with mean~$\excesspoparg{i,\baselearner',h'}$.
Therefore, the result follows by Hoeffding's inequality~\cite[Prop.~2.5]{wainwright19-a}.
\end{proof}
This result allows us to convert between test losses and population losses at the cost of a term that is typically negligible in comparison to the complexity terms.

With these tools and notations in place, we are ready to derive excess risk bounds.
The aim is to upper-bound the excess risk by an expression consisting of differences between training and test losses, for which we can apply our generalization bounds.
Starting from the excess risk on task~$\tau_0$, which is our fixed target task, we get
\begin{align}
\excesspoparg{0,\baselearner,\hat h} \!-\! \excesspoparg{0,\baselearner^*,h^*} \!&=\! \excesspoparg{0,\baselearner,\hat h} \!-\! \excesspoparg{0,\baselearner^*,\hat h} \!+\! \excesspoparg{0,\baselearner^*,\hat h} \!-\! \excesspoparg{0,\baselearner^*,h^*} \nonumber \\
&\leq \excesspoparg{0,\baselearner,\hat h} - \excesspoparg{0,\baselearner^*,\hat h} + D.
\end{align}
Here,~$D=\sup_{\tau_0}\excesspoparg{0,\baselearner^*,\hat h} - \excesspoparg{0,\baselearner^*,h^*}$ is the worst-case representation difference~\cite{tripuraneni-20a}, which we will later bound using a task diversity assumption.
Next, by Lemma~\ref{lem:hoeffding}, with probability at least~$1-2\delta$,
\begin{align}
\excesspoparg{0,\baselearner,\hat h} - \excesspoparg{0,\baselearner^*,\hat h} + D \leq \excesstestarg{0,\baselearner,\hat h} - \excesstestarg{0,\baselearner^*,\hat h} + D + 2\sqrt{\frac{\log \frac{2}{\delta}}{m}}.
\end{align}
Next, we use the risk decomposition
\begin{align}
&\excesstestarg{0,\baselearner,\hat h} - \excesstestarg{0,\baselearner^*,\hat h} \\
=&\excesstestarg{0,\baselearner,\hat h} - \excesstrainarg{0,\baselearner,\hat h} +\excesstrainarg{0,\baselearner,\hat h} - \excesstrainarg{0,\baselearner^*,\hat h} + \excesstrainarg{0,\baselearner^*,\hat h} -   \excesstestarg{0,\baselearner^*,\hat h} \\
\leq & \excesstestarg{0,\baselearner,\hat h} - \excesstrainarg{0,\baselearner,\hat h} + \excesstrainarg{0,\baselearner^*,\hat h} -   \excesstestarg{0,\baselearner^*,\hat h},
\end{align}
where the last step follows because~$\excesstrainarg{0,\baselearner,\hat h} \leq \excesstrainarg{0,\baselearner^*,\hat h} $, since~$\baselearner$ is an empirical risk minimizer.
Notice that the resulting expression is the difference between test and training losses on task~$\tau_0$ for two different algorithms.
These terms are simply the generalization gaps for a conventional learning setting.
These terms can be bounded by applying Corollary~\ref{cor:rep-learning-high-probability}, but for the case where~$\metan=1$ and~$\mathcal H=\{\hat h\}$, which implies that~$d_N=0$.
We conclude that there exists a constant~$C_1$ such that, with probability at least~$1-\delta$,
\begin{align}
\excesstestarg{0,\baselearner,\hat h} \!-\! \excesstrainarg{0,\baselearner,\hat h} \!+\! \excesstrainarg{0,\baselearner^*,\hat h} \!-\!   \excesstestarg{0,\baselearner^*,\hat h} \leq C_1\sqrt{ \frac{ \dVC \log\lefto(\frac{\sqrt m}{\dVC}\righto) \!+\! \log\lefto(\frac{\sqrt m}{\delta}\righto) }{m} }.
\end{align}
It remains to bound~$D$.
First, by the task diversity assumption,
\begin{align}
D&=\sup_{\tau_0}\excesspoparg{0,\baselearner^*,\hat h} - \excesspoparg{0,\baselearner^*,h^*}\\
&\leq \nu^{-1} \left( \excesspoparg{-\oneton,\baselearner^*,\hat h} -\excesspoparg{-\oneton,\baselearner^*,h^*} \right) + \epsilon.
\end{align}
We note here that, while the way that \cite{tripuraneni-20a} uses the assumption of task diversity requires that the difference between the minimum population losses for task~$\tau_0$ based on~$\hat h$ and~$h^*$ is controlled by the corresponding risks for tasks~$\oneton$, i.e., the tasks upon which~$\hat h$ is chosen, we instead assume that it is controlled by the corresponding losses for tasks~$-\oneton$, i.e., tasks that are independent from~$\hat h$.
In this sense, the diversity assumption that we use is arguably weaker.

By a risk decomposition, we get
\begin{align}
&\excesspoparg{-\oneton,\baselearner^*,\hat h} -\excesspoparg{-\oneton,\baselearner^*,h^*} \\
=& \excesspoparg{-\oneton,\baselearner^*,\hat h} -\excesspoparg{-\oneton,\baselearner,\hat h} + \excesspoparg{-\oneton,\baselearner,\hat h} -\excesspoparg{-\oneton,\baselearner^*,h^*} \\
\leq &\excesspoparg{-\oneton,\baselearner,\hat h} -\excesspoparg{-\oneton,\baselearner^*,h^*} ,
\end{align}
where the last step follows since~$\excesspoparg{-\oneton,\baselearner^*,\hat h} \leq\excesspoparg{-\oneton,\baselearner,\hat h}$.
By Lemma~\ref{lem:hoeffding}, with probability~$1-2\delta$,
\begin{align}
\excesspoparg{-\oneton,\baselearner,\hat h} -\excesspoparg{-\oneton,\baselearner^*,h^*} \leq \excesstestarg{-\oneton,\baselearner,\hat h} -\excesstestarg{-\oneton,\baselearner^*,h^*} + 2\sqrt{\frac{\log\frac2\delta  }{2n}}.
\end{align}
By a risk decomposition, we find that
\begin{align}
&\excesstestarg{-\oneton,\baselearner,\hat h} -\excesstestarg{-\oneton,\baselearner^*,h^*} \nonumber\\
\leq  &\excesstestarg{-\oneton,\!\baselearner,\hat h} \!-\! \excesstrainarg{\oneton,\!\baselearner,\hat h} \!+\! \excesstrainarg{\oneton,\!\baselearner,\hat h} \!-\! \excesstrainarg{\oneton,\!\baselearner^*,h^*} 
\!+\! \excesstrainarg{\oneton,\!\baselearner^*,h^*} \!-\!\excesstestarg{-\oneton,\!\baselearner^*,h^*} \nonumber\\
\leq &\excesstestarg{-\oneton,\baselearner,\hat h} \!-\! \excesstrainarg{\oneton,\baselearner,\hat h}  
\!+\! \excesstrainarg{\oneton,\baselearner^*,h^*} \!-\!\excesstestarg{-\oneton,\baselearner^*,h^*},
\end{align}
where the last step follows since~$\excesstrainarg{\oneton,\baselearner,\hat h} \leq \excesstrainarg{\oneton,\baselearner^*,h^*} $.
Now, notice that the resulting expression consists of the differences between the unobserved test losses and observed training losses for two different learning algorithms.
This means that we can apply Corollary~\ref{cor:rep-learning-high-probability} to find that there exists a constant~$C_2$ such that, with probability at least~$1-\delta$,
\begin{multline}
\excesstestarg{-\oneton,\baselearner,\hat h} \!-\! \excesstrainarg{\oneton,\baselearner,\hat h}  
\!+\! \excesstrainarg{\oneton,\baselearner^*,h^*} \!-\!\excesstestarg{-\oneton,\baselearner^*,h^*}\\
\leq  C_2 \sqrt{ \frac{d_N \log \left(\binom{N}{2}\frac{\taskin\metan}{d_N} \right) + \metan\dVC \log \left(\frac{\taskin}{\dVC} \right) + \log\lefto( \frac{\sqrt{\taskin\metan}}{\delta}\righto)}{\taskin\metan} }.
\end{multline}
Thus, by putting it all together, using a union bound to combine the probabilistic inequalities, we find that there exists constants~$C_1,C_2,C_3$ such that, with probability at least~$1-\delta$,
\begin{multline}
\excesspoparg{0,\baselearner,\hat h} \!-\! \excesspoparg{0,\baselearner^*,h^*} 
\leq C_1\sqrt{ \frac{ \dVC \log\lefto(\frac{\sqrt m}{\dVC}\righto) \!+\! \log\lefto(\frac{\sqrt m}{\delta}\righto) }{m} } \\
+ C_2 \nu^{-1}\sqrt{ \frac{d_N \log \left(\binom{N}{2}\frac{\taskin\metan}{d_N} \right) + \metan\dVC \log \left(\frac{\taskin}{\dVC} \right) + \log\lefto( \frac{\sqrt{\taskin\metan}}{\delta}\righto)}{\taskin\metan} } + \epsilon,
\end{multline}
where we note that the penalty terms arising from the union bound and converting between test and population losses have been absorbed using the constants.

\end{proof}

Thus, under the assumption of task diversity, we obtained an excess risk bound for a fixed target task, as was done in~\cite{tripuraneni-20a}.
However, if we are interested in bounding the excess risk for a new, randomly drawn task, rather than a fixed target, task diversity is not necessary.
In the following corollary, we demonstrate this by deriving an excess risk bound with respect to the population loss for a new, random task.
While we only present a bound based on Corollary~\ref{cor:rep-learning-high-probability}, similar excess risk bounds can be derived for the average case and from the other high-probability bounds.

\begin{cor}\label{cor:excess-random-task}
Consider the setting of Corollary~\ref{cor:rep-learning-high-probability}.
Assume that~$\baselearner$ is an empirical risk minimizer, that~$\baselearner^*$ is an oracle algorithm, and let
\begin{align}
\hat h &\in \argmin_{h'} \excesstrainarg{\oneton,\baselearner,h'},\\
    h^* &\in \argmin_{h'} \excesspoparg{-i,\baselearner^*,h'}.
\end{align}
Then, there exists a constant~$C$ such that, with probability at least~$1-\delta$ under~$\pacbtraindata$,
\begin{equation}
\excesspoparg{-i,\baselearner,\hat h}\!-\!\excesspoparg{-i,\baselearner^*,h^*} 
\!\leq\! C\sqrt{ \frac{d_N \log \lefto(\binom{N}{2}\frac{\taskin\metan}{d_N} \righto) \!+ \!\metan\dVC \log \lefto(\frac{\taskin}{\dVC} \righto) \!+\! \log\lefto( \frac{\sqrt{\taskin\metan}}{\delta}\righto)}{\taskin\metan} }.
\end{equation}
\end{cor}
\begin{proof}
We begin with the risk decomposition
\begin{align}
\excesspoparg{-i,\baselearner,\hat h}\!-\!\excesspoparg{-i,\baselearner^*,h^*} \! &=\! \excesspoparg{-i,\baselearner,\hat h} \!-\! \excesstrainarg{\oneton,\baselearner,\hat h} \!+\! \excesstrainarg{\oneton,\baselearner,\hat h} \\
&\quad - \excesstrainarg{\oneton,\baselearner^*,h^*} + \excesstrainarg{\oneton,\baselearner^*,h^*} - \excesspoparg{-i,\baselearner^*,h^*} \nonumber\\
&\leq \!\excesspoparg{-i,\baselearner,\hat h} \!-\! \excesstrainarg{\oneton,\baselearner,\hat h} \!+\! \excesstrainarg{\oneton,\baselearner^*,h^*\!} \!-\! \excesspoparg{-i,\baselearner^*,h^*\!} ,\nonumber
\end{align}
where we used that~$\excesstrainarg{\oneton,\baselearner,\hat h} \leq \excesstrainarg{\oneton,\baselearner^*,h^*} $.
Next, by Lemma~\ref{lem:hoeffding},
\begin{multline}
\excesspoparg{-i,\baselearner,\hat h} \!-\! \excesstrainarg{\oneton,\baselearner,\hat h} \!+\! \excesstrainarg{\oneton,\baselearner^*,h^*} \!-\! \excesspoparg{-i,\baselearner^*,h^*}\\
\leq \excesstestarg{-\oneton,\baselearner,\hat h} \!-\! \excesstrainarg{\oneton,\baselearner,\hat h} \!+\! \excesstrainarg{\oneton,\baselearner^*,h^*} \!-\! \excesstestarg{-\oneton,\baselearner^*,h^*} + 2\sqrt{\frac{\log (2\delta)}{2n\hat n}}.
\end{multline}
This expression consists of the differences between the unobserved test losses and observed training losses for two different learning algorithms.
We can thus use Corollary~\ref{cor:rep-learning-high-probability} to conclude that there exists a constant~$C$ such that, with probability at least~$1-\delta$,
\begin{multline}
\excesstestarg{-\oneton,\baselearner,\hat h} \!-\! \excesstrainarg{\oneton,\baselearner,\hat h} \!+\! \excesstrainarg{\oneton,\baselearner^*,h^*} \!-\! \excesstestarg{-\oneton,\baselearner^*,h^*} + 2\sqrt{\frac{\log (2\delta)}{2n\hat n}} \\
\leq  C\sqrt{ \frac{d_N \log \lefto(\binom{N}{2}\frac{\taskin\metan}{d_N} \righto) + \metan\dVC \log \lefto(\frac{\taskin}{\dVC} \righto) + \log\lefto( \frac{\sqrt{\taskin\metan}}{\delta}\righto)}{\taskin\metan} }.
\end{multline}
Here, the penalty term from the conversion between population and test loss has been absorbed into the constant~$C$.
From this, the desired result follows.
\end{proof}

\end{document}